\documentclass[10pt,journal,compsoc]{IEEEtran}

\ifCLASSOPTIONcompsoc
  \usepackage[nocompress]{cite}
  
\else
  \usepackage{cite}
\fi
\usepackage{booktabs}
\usepackage{amsmath}
\usepackage{balance}
\usepackage{enumitem}
\usepackage{amsthm}
\usepackage{hyperref}
\usepackage{amssymb}
\usepackage{colortbl}
\usepackage{stfloats}
\usepackage{xcolor} 
\usepackage{textcomp}
\usepackage{graphicx}
\usepackage{verbatim}
\usepackage{multirow}
\usepackage{algorithm}
\usepackage{multicol}
\usepackage{algpseudocode}
\usepackage{mathtools}
\newtheorem{theorem}{Theorem}
\newtheorem{lemma}[theorem]{Lemma}
\newtheorem{definition}{Definition}
\ifCLASSINFOpdf
  
\else
  
\fi

\begin{document}

\title{Federated Learning with Differential Privacy: An Utility-Enhanced Approach}

\author{Kanishka Ranaweera,  Dinh C. Nguyen,~\IEEEmembership{Member, ~IEEE, } Pubudu N. Pathirana,~\IEEEmembership{Senior Member, ~IEEE, } David Smith, Ming Ding,~\IEEEmembership{Senior Member, ~IEEE, } Thierry Rakotoarivelo,~\IEEEmembership{Senior Member, ~IEEE, } Aruna Seneviratne,~\IEEEmembership{Senior Member, ~IEEE }

\IEEEcompsocitemizethanks{\IEEEcompsocthanksitem Kanishka Ranaweera is with School of Engineering and Built Environment, Deakin University, Waurn
Ponds, VIC 3216, Australia, and also with the Data61, CSIRO, Eveleigh, NSW 2015, Australia. \protect\\
E-mail: kranaweera@deakin.edu.au
\IEEEcompsocthanksitem Dinh C. Nguyen is with the Department of Electrical and Computer Engineering,  The University of Alabama in Huntsville Alabama, USA. \protect\\
E-mail: Dinh.Nguyen@uah.edu
\IEEEcompsocthanksitem Pubudu N. Pathirana is with School of Engineering and Built Environment, Deakin University, Waurn Ponds, VIC 3216, Australia. \protect\\
E-mail: pubudu.pathirana@deakin.edu.au
\IEEEcompsocthanksitem David Smith is with Data61, CSIRO, Eveleigh, NSW 2015, Australia. \protect\\
E-mail: david.smith@data61.csiro.au
\IEEEcompsocthanksitem Ming Ding is with Data61, CSIRO, Eveleigh, NSW 2015, Australia. \protect\\
E-mail: ming.ding@data61.csiro.au
\IEEEcompsocthanksitem Thierry Rakotoarivelo is with Data61, CSIRO, Eveleigh, NSW 2015, Australia. \protect\\
E-mail: thierry.rakotoarivelo@data61.csiro.au
\IEEEcompsocthanksitem Aruna Seneviratne is with School of Electrical Engineering and Telecommunications, University of New South Wales (UNSW), NSW, Australia. \protect\\
E-mail: a.seneviratne@unsw.edu.au

}
}

\markboth{}%
{}

\IEEEtitleabstractindextext{%
\begin{abstract}
Federated learning has emerged as an attractive approach to protect data privacy by eliminating the need for sharing clients' data while reducing communication costs compared with centralized machine learning algorithms. 
However, 
recent studies have shown that federated learning alone does not guarantee privacy, 
as private data may still be inferred from the uploaded parameters to the central server. 
In order to successfully avoid data leakage, 
adopting differential privacy (DP) in the local optimization process or in the local update aggregation process has emerged as two feasible ways for achieving sample-level or user-level privacy guarantees respectively, in federated learning models. 
However, 
compared to their non-private equivalents, 
these approaches suffer from a poor utility. 
To improve the privacy-utility trade-off, 
we present a modification to these vanilla differentially private algorithms based on a Haar wavelet transformation step and a novel noise injection scheme that significantly lowers the asymptotic bound of the noise variance. 
We also present a holistic convergence analysis of our proposed algorithm, 
showing that our method yields better convergence performance than the vanilla DP algorithms. 
Numerical experiments on real-world datasets demonstrate that our method outperforms existing approaches in model utility while maintaining the same privacy guarantees. 
\end{abstract}

}

\maketitle

\IEEEdisplaynontitleabstractindextext

\IEEEpeerreviewmaketitle

\IEEEraisesectionheading{\section{Introduction}\label{sec:introduction}}

\IEEEPARstart{T}{he} proliferation of Internet of Things (IoT) devices has led to exponential growth in data creation. 
Machine learning (ML) has become an essential tool to analyze this data and extract valuable insights for various applications, 
including facial recognition, 
data analytics, 
weather prediction, 
and speech recognition, among others~\cite{guo2019survey, ayele2018air, nguyen20216g, kumar2021mobihisnet, ma2019privacy}. 
However, 
in real-world settings, data --- particularly personal data --- is often created and stored on end-user devices. 
The majority of traditional ML algorithms require the centralization of these training data, 
which involves collecting and processing data at a potent cloud-based server \cite{verbraeken2020survey, nguyen2021federated}. 
This process carries significant risks to data integrity and privacy, 
particularly when it comes to personal data. 
This is because personal data may contain sensitive or private information, 
and the centralization of data creates a single point of failure that might jeopardize the integrity of the data. 
Exposing such data can lead to financial loss, lawsuits, reputation damage, and even physical violence. 
In addition, 
centralized data processing and administration impose restricted transparency and provenance, 
which could result in end-users losing trust in the system~\cite{lim2020federated}. 

Federated learning (FL) \cite{konevcny2016federated, mcmahan2017communication} is a cutting-edge development in distributed ML, 
designed to enable collaborative learning in decentralized environments. 
In this context, 
an ML model is implemented by applying an algorithm to a variety of local datasets maintained by client devices. 
The updated model parameters are then sent to a central server for aggregation,
eliminating the need for gathering and processing training data at a centralized data server.

In a traditional machine learning system, 
a large dataset is partitioned across cloud servers and optimized using techniques such as stochastic gradient descent (SGD)~\cite{bottou2010large}, 
adaptive moment estimation (ADAM)~\cite{kingma2014adam}, 
or adaptive gradient algorithm (Adagrad)~\cite{duchi2011adaptive}. 
However, 
in an FL environment, 
the data is dispersed unevenly across millions of devices with varying degrees of availability and connectivity. 
These bandwidth and latency restrictions lead to the development of Google's Federated Averaging (FedAvg) algorithm [10], 
which uses far fewer communication rounds to train deep networks than a naive version of a federated machine learning optimizer. 

The key concept behind FL is to compute local updates using the powerful CPUs available on client devices, 
allowing for training to consume substantially less communication overhead. 
This is achieved by generating high-quality updates in fewer cycles than traditional gradient steps, 
resulting in a better model.

\color{black}
With the increasing need for data security and privacy in big data applications and distributed learning systems, the preservation of privacy has become a major concern. FL provides a significant advantage by performing local training without transferring personal data between the server and clients, thus safeguarding client data from covert attackers. However, recent research has shown that variations in parameters trained and uploaded by clients can still reveal private information to some extent \cite{geiping2020inverting, zhu2019deep}.

Differential privacy (DP) \cite{dwork2006,dwork2014algorithmic} has shown to be a promising solution for preventing adversaries from inferring information from FL applications. This DP characteristic confers strong privacy  on a data-processing mechanism rather than the data itself. It is widely used for query release, synthetic data synthesis, and machine learning training, among other things \cite{dwork2006, asghar2021, liu2021}. The work in \cite{abadi2016deep} introduced a novel optimization technique based on the concept of DP \cite{dwork2014algorithmic} applicable to deep learning techniques, which can also be adapted for use at the sample level in FL algorithms. More specifically, they presented a modification to the classic SGD technique, which is a cost function-based optimization algorithm commonly used in DL applications to fit linear classifiers and regressors. This algorithm is called differentially private stochastic gradient descent (DP-SGD) and is considered highly effective in achieving sample-level differential privacy in FL applications. However, users may occasionally provide several examples to the training dataset, in which case sample-level DP may not be strong enough to prevent user data from being memorized. To address this issue, \cite{mcmahan2017learning} created the Differentially Private Federated Averaging (DP-FedAvg) algorithm, which provides user-level DP. This means that the output distribution of models remains constant even if we add or delete all of the training samples from any one client. However, these DP algorithms \cite{abadi2016deep,mcmahan2017learning} exhibit a considerable loss in model utility when compared to non-private versions. The insertion of Gaussian noise into the gradient steps utilized by DP-SGD and the insertion of Gaussian noise into the global aggregation in DP-FedAvg cause this deterioration. Increasing the amount of noise increases privacy assurances, but it also reduces utility.
\color{black}

\color{black}Motivated by these limitations, our paper presents a novel modification to the vanilla DP-SGD and DP-FedAvg algorithms that delivers significantly enhanced utility while maintaining the same privacy guarantees, resulting in a superior utility-privacy tradeoff. This approach is inspired by the Privelet framework introduced by \cite{xiao2010differential} for differentially private data publishing of range count queries. However, while Privelet is designed for static, non-learning scenarios, our work extends this technique to dynamic and iterative machine learning and FL environments. In these settings, we address unique challenges such as managing the sensitivity of gradient updates and ensuring robust differential privacy guarantees across multiple rounds of communication. Our method integrates wavelet transforms into the differentially private federated averaging process, providing theoretical convergence guarantees and practical performance improvements demonstrated through extensive empirical validation on real-world FL datasets.\color{black}The strategy is based on a Haar wavelet transformation (HWT), a linear transformation approach commonly used in data compression, image, and signal processing \cite{santoso1997power, zhang2019wavelet, bentley1994wavelet}. \cite{xiao2010differential} demonstrated that by performing Laplace noise injection in the resultant coefficients of HWT, the noise variance can be greatly reduced, hence enhancing the accuracy of information obtained from DP range-count queries. Our method proposes a new Gaussian noise injection system for DP-SGD and DP-FedAvg that adapts and employs the HWT technique.

\textbf{Contributions and Paper Structure:} Our paper provides the following key contributions:

\begin{enumerate}[topsep=0pt, itemsep=-1ex, partopsep=1ex, parsep=1ex]
\item We propose a novel modification to the DP-SGD and DP-FedAvg algorithms by employing a Haar wavelet transform approach to achieve a better utility-to-privacy tradeoff than their vanilla counterparts (i.e., ensuring the same $(\epsilon, \delta)$ DP guarantees for much better utility). 
\item We provide theoretical proof to show that our method significantly
outperforms the baseline method in \cite{abadi2016deep}, across the noise variance bound. This is demonstrated to provide higher accuracy.
\item We present a convergence analysis to show that our method yields better convergence performance than the vanilla DP-FedAvg and DP-SGD algorithms.
\item We demonstrate that when applied to two widely known scientific benchmark datasets, our technique greatly outperforms the vanilla DP-SGD and DP-FedAvg in terms of both accuracy and loss.
\item We provide experimental results, showing that our technique works with other optimization algorithms such as DP-ADAM and DP-AdaGrad in a "black-box" way to obtain a better utility-to-privacy trade-off.
\end{enumerate}

The remainder of this paper is organized as follows. The background principles that are utilised in our contribution are briefly discussed in Section~\ref{sec:bg}. Section~\ref{sec:so} discusses our system model and the privacy model. Section~\ref{sec:met} looks into the specifics of our new method. Our analytical results are explained in Section~\ref{sec:ta}, and our experimental assessments are presented in Section~\ref{sec:ex}. Section~\ref{sec:co} concludes the paper and points out potential future work.

\begin{table}[]
\centering
\caption{nomenclature}
\label{nomenclature}
\begin{tabular}{@{}ll@{}}
\toprule
$\mathcal{M}$ & A randomized mechanism \\ 
$\varepsilon, \delta$ & DP privacy parameters \\
$D$, $D'$ & Two adjacent datasets \\
${x}_i$ & $i^{th}$Subset of training data \\
${x}_i^k$ & $i^{th}$ Subset of $k^{th}$ clients' training data \\
L($\theta,x_i$) & Loss function \\ 
$\theta$ & Model parameters \\
$\eta_{t}$ & Learning rate \\
\color{black}$K_t$ & \color{black}Total number of clients \\
\color{black}$K$ & \color{black}Number of clients sampled per round \color{black}\\
$k$ & Client indexes \\
$w_t$ & $t^{th}$ weight parameter broadcast \\
 & to the clients by the server \\
$w_{t}^k$ & Weight parameters trained at the $k^{th}$ \\
 & client in the $t^{th}$ communication round \\  
$d$ & Size of the collective client dataset\\
$d_k$ & Size of the dataset held by $k^{th}$ client\\
$\mathcal{S}_f$ & Bounded sensitivity \\
$q_c$ & Client selection probability \\
$T$ & Total number of communication rounds \\
$C,C_j$ & Flat/per-layer clipping threshold \\
$\mathcal{N}(a,b^2)$ & Gaussian noise with mean $a$ and \\
 & standard deviation $b$ \\
$W_{(Haar)}$ & Wavelet transform weight function \\
$L_k(\theta)$ & Loss of client k \\
$\theta^*$ & Minimizer of function $L(\theta)$ \\
\bottomrule
\end{tabular}
\end{table}

\section{Background}\label{sec:bg}
This section provides a brief overview of the fundamentals of FL, the DP algorithm, and the instances where DP is used to make FL algorithms private.

\subsection{Differential Privacy}

\begin{figure}[!t]
\centering
\includegraphics[width=3.3in]{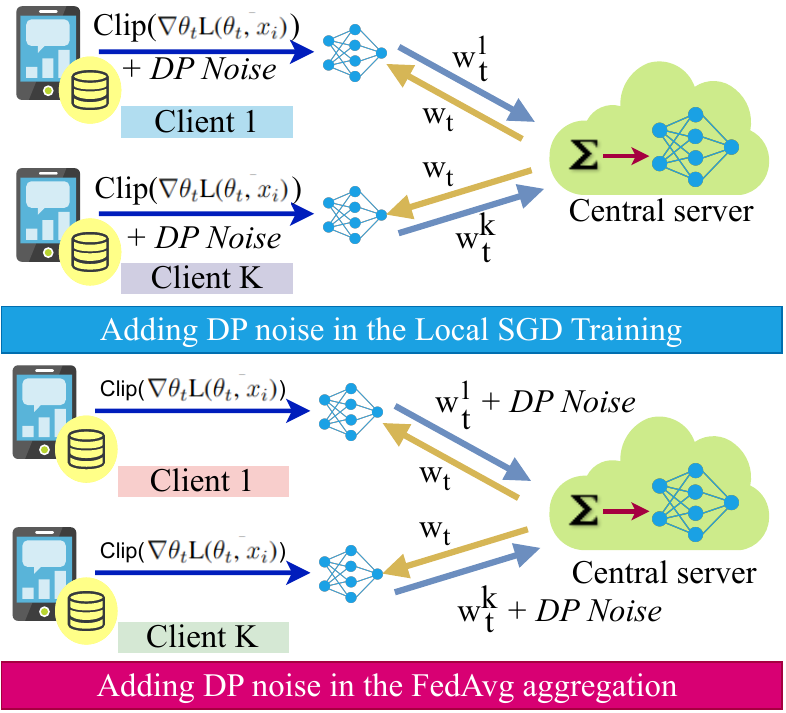}
\caption{FL configurations with sample-level DP and user-level DP.}
\label{fedavg1}
\end{figure}

In the literature, differential privacy (DP) has been widely proven as a robust privacy protection mechanism \cite{dwork2008differential}, \cite{dwork2014algorithmic}. In its most basic form, the goal of DP is to ensure that nothing further can be inferred about an individual data record despite the inclusion of that data record in the mechanism's input, a concept known as plausible deniability. The capacity to quantify the privacy budget of a DP mechanism using parameters ($\varepsilon$, $\delta$)\footnote{Although this is technically characterized as approximate differential privacy because of the presence of the $\delta$ parameter, the great majority of literature simply refers to such with the $\delta$ parameter as differential privacy, as we do here.} is also a significant aspect of DP.

\begin{definition}
A randomized algorithm $\mathcal{M}$ satisfies ($\varepsilon, \delta$) differential privacy for two adjacent datasets $D$ and $D'$ and any subset of output S$\subseteq$ Range(R) holds, 

\begin{equation}
\label{DP_def}
    Pr[\mathcal{M}(D)\in S] \leq e^{\varepsilon}Pr[\mathcal{M}(D')\in S] + \delta.
\end{equation}
\end{definition}

Here, $\varepsilon$ is a measure of privacy loss when a data point, such as a record from a dataset, is added or removed, with a smaller $\varepsilon$ indicating a greater privacy guarantee. The term $\delta$ is the failure probability, which is the maximum amount by which the algorithm is allowed to deviate from the desired privacy guarantee. For example, in \ref{DP_def}, we can achieve $\varepsilon$ privacy guarantee with $(1 - \delta)$ probability.  and $\delta$ is the failure probability. 

For a given function $f(x)$, DP can be achieved by adding noise calibrated to the function's sensitivity $S_f$. There are several noise mechanisms that can be employed to make $f(x)$ differentially private, including the Laplace mechanism \cite{dwork2008differential} and the exponential mechanism \cite{mcsherry2007mechanism}. However, in approximate DP, a classic DP technique \cite{dwork2014algorithmic} is to draw noise using a Gaussian mechanism. Such a differentially private counterpart $F(x)$ for function $f(x)$ is defined as follows:

\begin{equation}
\begin{split}
F(x) = f(x) + \mathcal{N}(0, \sigma^{2}\mathrm{S}_f), 
\end{split}
\end{equation}
where $\sigma$ is the scale of the noise added to the data and controls the trade-off between privacy and data utility. A larger value of $\sigma$ results in less noisy data and better utility, but weaker privacy protection. On the other hand, a smaller value of $\sigma$ results in more noisy data and lower utility, but stronger privacy protection.

\subsection{Federated Learning With Differential Privacy}

The previously-mentioned issues with conventional FL systems, i.e. the possibility of inferring sensitive information from the uploaded model parameters, can be addressed by adding DP to the FL framework, which allows for the training of models on distributed data while preserving the privacy of the data on each device. There are two main techniques for adding differential privacy to FL: DP-SGD~\cite{abadi2016deep} and DP-FedAvg\cite{mcmahan2017learning}.

In DP-SGD, noise is added to the gradients computed on the local data during the training process, ensuring that the updates sent to the central server do not reveal information about the individual data points used for training. This technique provides stronger privacy guarantees, making it suitable for scenarios where data is highly sensitive, such as in medical or financial applications~\cite{jiang2021differential}. However, DP-SGD can also be computationally expensive, as the noise must be added to each gradient computation, increasing the training time on each device \cite{NEURIPS2021_a3842ed7}. The weight updating of the DP-SGD algorithm in the local training process can be summed up as follows. 

\begin{equation}
\label{eq:update}
\theta^{t+1}= \theta^t - \eta_{t}( \triangledown_{\theta}\mathcal{L}(\theta , x_i) + \mathcal{N}(0, \sigma^{2}\mathrm{S}_fI)),
\end{equation}
where $\mathcal{S}_f$ represents a sensitivity bound for the gradients from a deep neural network. Since there is no prior bound to regulate the sensitivity of the gradients from a deep neural network such that $\Vert \triangledown_{\theta}\mathcal{L}(\theta, x_i)\Vert \le \mathcal{S}_f$, bounding the sensitivity before adding Gaussian noise is necessary, In order to enforce such a bound the updates needs to be clipped. \cite{mcmahan2017learning} proposes the following clipping techniques to accomplish this.

\begin{enumerate}[topsep=0pt, itemsep=-1ex, partopsep=1ex, parsep=1ex]
\item \textbf{Flat clipping}: Gradient concatenation of all the layers in the neural network is clipped as $g_{t}$=$g_{t}$/( max(1, $\dfrac{{\Vert g_{t})\Vert}_{2}}{C}$)) given an overall clipping parameter $C$. ($\mathcal{S}_f \leftarrow C$)

\item \textbf{Per layer clipping}: Given a per-layer clipping value $C_j$, each layer j is clipped separately as $g_{t}(j)$=$g_{t}(j)$/( max(1, $\dfrac{{\Vert g_{t}(j))\Vert}_{2}}{C_j}$)) since the updates of each layer in a deep network may have significantly different $L_2$ norms as. \color{black} Here $p$ being the number of layers in the deep neural network, let $\mathcal{S}_f = \sqrt{\sum_{j=1}^pC_j}$
\end{enumerate}

\color{black}In DP-FedAvg, noise is added during the aggregation of the model updates from each client to the central server, providing computational efficiency as the noise is injected only once during the aggregation step. This process offers strong privacy guarantees, not just at the local data level but also at the global model level. Specifically, adding noise to the local SGD training protects the privacy of individual data points at the client level by preventing information leakage from local updates. However, DP-FedAvg additionally protects the privacy of the global model, as the noise added at the aggregation step ensures that the final model remains differentially private. This is particularly beneficial in scenarios where adversaries may attempt to infer sensitive information from the global model parameters.\color{black} On the other hand, DP-SGD allows for more flexible privacy budget allocation and may have lower computational overhead than DP-FedAvg in certain scenarios. It is also better suited to small-scale Federated Learning where there are only a few devices participating in the learning process. In addition, DP-SGD is a simpler algorithm that can be easier to implement and understand in some cases. Ultimately, the choice between DP-SGD and DP-FedAvg depends on the specific requirements of the application, such as the level of privacy needed, the available computational resources, and the nature of the data being used. 

In a typical FL system, with one server and $K$ clients participating in each round of training. Learning a model from data stored at the $K$ number of linked client devices is the aim of the server. Formally, the weights $w_{t}^k$ received from the $K$ clients (chosen with  a probability of $q_c$) gets aggregated at the central server as follows, 

\begin{equation}
\label{fedavgalgo}
    w_t = \sum_{k=1}^K \dfrac{d_k}{q_cd}w_{t+1}^k.
\end{equation}

However, when using this approach with DP, the sensitivity  $\mathcal{S}_f$ of the FL process needs to be bounded, as different clients may produce samples of different sizes. By applying eq.~\ref{fedavgalgo} and considering that the norm of ${\Vert \eta_j w_{t+1}^k \Vert}_2 \le C$ is bounded, we can express the sensitivity bound of the client update using weighted averaging as follows,

\begin{equation}
\mathcal{S}_f \le \dfrac{C}{q_cd}.
\end{equation}

To protect the privacy of data on each device, it is essential to monitor the privacy budget utilized during the training process. A method called the Moments Accountant (MA), introduced by \cite{abadi2016deep}, offers an upper bound for the privacy curve of a composition of differential privacy (DP) algorithms. The privacy loss analysis using the Moments Accountant involves combining the privacy curve of each training iteration with itself $T$ times, where $T$ is the total number of training iterations. The Moments Accountant has been integrated into the Renyi Differential Privacy (RDP) accountant framework introduced by \cite{Mironov_2017}. Although the running time of this accountant is independent of $T$, it only provides an upper bound and cannot approximate the privacy curve to arbitrary accuracy. \cite{dong2019gaussian,bu2020deep} proposed the idea of Gaussian Differential Privacy (GDP) and developed an accountant for DP algorithms based on the central limit theorem. This accountant provides an approximation to the actual privacy curve and the approximation improves with $T$.

\section{Scheme Overview}\label{sec:so}
\subsection{System Model}

\color{black}Our system model is based on an FL architecture, where a central server coordinates the learning process with 
$K_t$ representing the total number of participating clients. In each communication round, a subset of clients, denoted as $K$, is randomly sampled to participate.\color{black} To provide a concise summary of the FL system, we outline the following key steps:

\begin{enumerate}[topsep=0pt, itemsep=-1ex, partopsep=1ex, parsep=1ex]
\item \textbf{Initialization}: The server first randomly chooses or pretrains a global model using public data. 
\item \textbf{Broadcasting}: The global model's parameters are distributed to the clients by the server and the local models get updated by the global model.
\item \textbf{Local training}: Each client updates
their local models using their data and uploads the updated model parameters to the server.
\item \textbf{Aggregation}: The receiving models are weighted and aggregated by the central server, which then transmits the updated models back to the nodes.
\item \textbf{Termination}: All phases, with the exception of \textbf{\textit{Initialization}}, are iterated until a pre-set termination requirement is satisfied (e.g. the model accuracy exceeds a threshold or a maximum number of iterations has been attained)
\end{enumerate}

In our FL system, which is illustrated in Fig.~\ref{fl_model}, each client has access to its own local database $D_k$, where  $k \in (1, 2, ..., K)$. The server aims to learn a model from the data stored on the $K$ linked client devices. 

\begin{figure}[!t]
\centering
\includegraphics[width=3.3in]{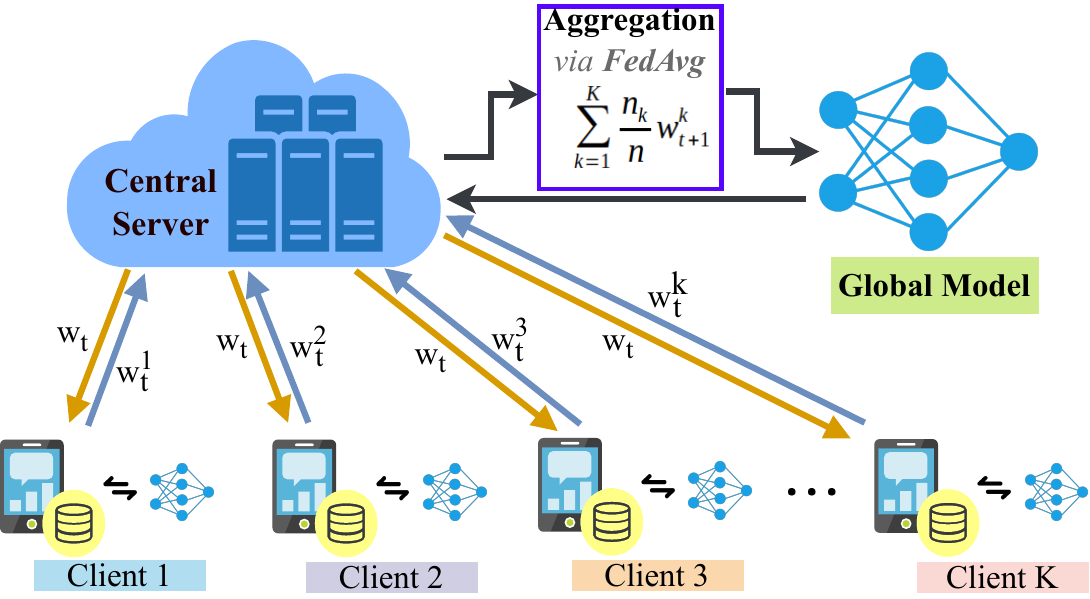}
\caption{Example of a federated learning model.}
\label{fl_model}
\end{figure}

\subsection{Threat Model}

Our threat model considers that the server is honest but curious about the data of the clients and that outside adversaries could intercept model parameters sent through the FL system. The intermediate parameters that are communicated with the server during the model aggregation process could reveal sensitive information about the data of clients, such as uncommon features or distinctive patterns that could be utilized by an attacker to deduce personal information. We take into account an opponent that actively seeks to extract private data from the gathered information. The opponent could be a user of the system, an outsider who has access to the system, or an intruder with the ability to undermine its security. We assume that the adversary is fully aware of the design, algorithm, and settings of the FL system.

In our proposed threat model, we identify two types of attacks that can compromise the privacy of clients' data in federated learning: membership inference attacks and model inversion attacks.

\begin{enumerate}[topsep=0pt, itemsep=-1ex, partopsep=1ex, parsep=1ex]

\item \textbf{Membership Inference Attacks}: This type of attacks attempt to determine if a specific client's data was used to train the model. It can reveal sensitive information about a client, such as their medical conditions or financial status. An attacker can perform this type of attack by observing the intermediate parameters shared with the server during the FL model aggregation process.
\item \textbf{Model Inversion Attacks}:
These attacks attempt to recover sensitive data used to train the model, such as individual medical records or financial transactions. It relies on analyzing the intermediate parameters and the final model to extract sensitive data. An attacker can use this information to infer details about the clients' data, such as rare features or unique patterns.
\end{enumerate}

To lessen these risks, we integrate DP approaches into our model aggregation procedure. This makes sure that neither the final model nor the intermediate parameters expose any private information about the clients. The trained model's performance, however, may be impacted by the local parameter vectors' inclusion of DP-injected noise. Hence, in order to increase classification performance while preserving the same amount of differential privacy, we strive to strengthen the model's robustness against DP-injected noise in this study. It is important to note that we consider two scenarios: adding DP noise to the SGD training process at the sample level and adding DP noise at FedAvg at the user level. We acknowledge that the broadcast of the global parameter vector and the shared intermediate parameters with the server may expose the confidential information of the clients to outside adversaries.

\section{Methodology}\label{sec:met}
This section outlines our novel modification to the DP-SGD algorithm to obtain sample level DP in FL settings, DP-SGD via wavelets (DP-SGD-WAV), and our suggested modification to the DP-FedAvg algorithm to obtain user-level DP, DP-FedAvg via Wavelets (DP-FedAvg-WAV), both by employing wavelet transforms. Additionally, we also go through the technique for computing the  Haar wavelet transform ( HWT) along with the novel noise injection scheme we propose.

\begin{algorithm}[h!]
  \caption{Pseudocode of DP-SGD via Wavelet Transform in a FL configuration}\label{dp-sgd}
  \begin{algorithmic}
    \Require {Input: K amount of clients chosen with a selection probability of $q_c \in (0, 1]$, $k^{th}$ clients local dataset $x_{i}^k$, loss function L($\theta$ , $x_{i}^k$)} Parameters: learning rate $\eta$, noise scale $\sigma$, local lot size $L_n$, coefficient norm bound C.
\Ensure $w_{t}$ and compute the overall privacy cost ($\varepsilon$, $\delta$) using a privacy accounting method.
\State \textbf{Server executes:}
\State \hspace{1.2em} Initialize $w_{t}$ randomly
\State \hskip1.5em \textbf{For} t $\in$ $\vert$T$\vert$ do
\State \hskip3em \textbf{For} each client k $\in$ $\vert$K$\vert$ do
\State \hskip4.5em $w_{t+1}^k \longleftarrow$ ClientTraining$(w_t, k)$

\State \hskip3em $w_{t+1} \longleftarrow \sum_{k=1}^K \dfrac{d_k}{q_cd}w_{t+1}^k$
\State \hrulefill
    \State \textbf{ClientTraining}($\theta_{n}$, k):
 \State \hspace{1.2em}\textbf{For} n $\epsilon$ $\vert$N$\vert$ do
\State \hspace{2.4em} Take a random sample $L_n$ with sampling 
\State \hspace{2.4em} probability q
\State \hspace{2.4em} \textbf{Compute gradient}
\State \hspace{2.4em} For each i $\in$ $L_{n}$, compute 
\State \hspace{3.6em} $g_n$($x_{i}^k$) $\longleftarrow$ $\nabla\theta_{n}$L($\theta_{n}$, $x_{i}^k$)
\State \hspace{3.6em} \textbf{Add noise}
\State \hspace{3.6em} $\widehat{g}_n$($x_{i}^k$) $\longleftarrow$ WaveletNoise($g_n(x_i^k)$)
\State \hspace{3.6em} \textbf{Descent}
\State \hspace{3.6em} $\theta_{n+1}$ $\longleftarrow$ $\theta_n$ - $\eta(\widehat{g}_n(x_{i}^k))$
\State \hspace{1.2em}\textbf{Return} $\theta_{n+1}$ to the server
\State \hrulefill
\State \textbf{WaveletNoise}($g_n(x_i^k$)):
\State \hspace{1.2em}\textbf{Wavelet transform}
\State \hspace{1.2em}$H_{n}(x_i^k) \longleftarrow g_{n}$($x_{i}^k$)
\State \hspace{1.2em}\textbf{Clip wavelet coefficients}
\State $H_{n}$($x_{i}^k$) $\leftarrow$ $H_{n}$($x_{i}^k$) / max(1, $\dfrac{{\Vert H_{n}(x_{i}^k)\Vert}_{2}}{C}$)
\State \hspace{1.2em}\textbf{For} each Haar $\in$ $H_n$($x_i^k$)
\State \hskip2,4em \textbf{Add noise}
\State \hskip2.4em $\widehat{H_n}(x_{i}^k) \longleftarrow$($H_{n}$(Haar)+$\mathcal{N}$(0, $\dfrac{\sigma^2C^2}{(W_{(Haar)})^2}$)
\State \hspace{1.2em}\textbf{Inverse wavelet transform}
\State $\widehat{g}_{n}(x_{i}^k) \longleftarrow \dfrac{1}{L}(\Sigma_{i} H_{n}(x_{i}^k)$)
\State \hspace{1.2em}\textbf{Return} $\widehat{g}_{n}(x_{i}^k)$
  \end{algorithmic}
\end{algorithm}

\subsection{DP-SGD via Wavelet Transforms}

We employ the exact same methodology as the standard DP-SGD. However, we suggest performing the clipping and noise injection on the wavelet coefficients, or final coefficients, of an HWT. Our method adds Gaussian noise calibrated to the clipping norm C to each wavelet coefficient with a variance of $(\dfrac{\sigma}{W_{(Haar)}})^2$, where $\sigma$ and $W_{(Haar)}$ denote the input and the weight function corresponding to each wavelet coefficient, respectively. In Section~\ref{sec::1d}, we go into greater detail about these weights and the HWT technique.

It has been intensely debated and extensively researched how to choose the best clipping threshold C for DP-SGD \cite{chen2020understanding}\cite{wei2021user}. In our trials, we employ a clipping strategy based on a median value and adhere to the methodology in \cite{abadi2016deep}. Once the noise is added to the coefficients, we utilize an inverse wavelet transformation step to construct the noisy gradient vectors. The neural network will then be updated in the opposite direction, just like in standard DP-SGD computations, using the reconstructed noisy gradients. Algorithm~\ref{dp-sgd} provides a summary of the improved DP-SGD version employing wavelet transformations.

\subsection{DP-FedAvg via Wavelet Transforms}

We demonstrate how our method can be modified to improve utility in user-level DP configurations, specifically by introducing a wavelet transformation function to perform noise injection at the DP-FedAvg algorithm's global aggregation step. However, given a required noise variance prior to training, the noise variance that is added to the wavelet coefficients will be adjusted to reflect this using a correlation that we provide in our theoretical results. The main difference between this and algorithm~\ref{dp-sgd} is that algorithm~\ref{dp-fed} is a synopsis of the proposed DP-FedAvg-Wav algorithm.
The primary difference between this and algorithm \ref{dp-sgd} is where the noise is added, which influences where we conduct the wavelet transformation. As previously stated, DP-FedAvg injects Gaussian noise during aggregation; thus, we perform wavelet transformation on the aggregation model parameters.

\begin{algorithm}
  \caption{Pseudocode of DP-FedAvg via Wavelet Transform}\label{dp-fed}
  \begin{algorithmic}
    \Require {Input: K amount of clients chosen with a selection probability of $q_c \in (0, 1]$, $k^{th}$ clients local dataset $x_{i}^k$, loss function L($\theta$ , $x_{i}$)} Parameters: learning rate $\eta$, noise scale $z$, local lot size $L_n$, gradient norm bound C.
    \Ensure $w_{t}$ and compute the overall privacy cost ($\varepsilon$, $\delta$) using a privacy accounting method.
\State \textbf{Server executes:}
\State \hspace{1.2em} Initialize $w_{t}$ randomly
\State \hskip1.5em \textbf{For} t $\in$ $\vert$T$\vert$ do
\State \hskip3em \textbf{For} each client k $\in$ $\vert$K$\vert$ do
\State \hskip4.5em $w_{t+1}^k \longleftarrow$ ClientTraining$(w_t, k)$
\State \hskip3em $w_{t+1} \longleftarrow \sum_{k=1}^K \dfrac{d_k}{q_cd}w_{t+1}^k$
\State \hskip3em $\sigma \longleftarrow \dfrac{zC}{q_cd}$
\State \hskip3em $\widehat{w}_{t+1} \longleftarrow$ WaveletNoise($w_{t+1},\sigma)$
\State \hrulefill
    \State \textbf{ClientTraining}($\theta_{n}$, k):
 \State \hspace{1.2em} \textbf{For} n $\epsilon$ $\vert$N$\vert$ do
\State \hspace{2.4em} Take a random sample $L_n$ with sampling 
\State \hspace{2.4em} probability q
\State \hspace{2.4em} \textbf{Compute gradient}
\State \hspace{2.4em} \textbf{For} each i $\in$ $L_{n}$, compute 
\State \hspace{3.6em} $g_n$($x_{i}^k$) $\longleftarrow$ $\nabla\theta_{n}$L($\theta_{n}$, $x_{i}^k$)
\State \hspace{3.6em} \textbf{Clip gradients}
\State \hspace{3.6em} $g_{n}$($x_{i}^k$) $\leftarrow$ $g_{n}$($x_{i}^k$) / max(1, $\dfrac{{\Vert g_{n}(x_{i}^k)\Vert}_{2}}{C}$)

\State \hspace{3.6em} \textbf{Descent}
\State \hspace{3.6em} $\theta_{n+1}$ $\longleftarrow$ $\theta_n$ - $\eta(\widehat{g}_n(x_{i}^k))$
\State \hspace{1.2em}\textbf{Return} $\theta_{n+1}$ to the server
\State \hrulefill
\State  \textbf{WaveletNoise}($w_t$,$\sigma$):
\State \hspace{1.2em} \textbf{Wavelet transform}
\State \hspace{1.2em} $H_{t} \longleftarrow w_{t}$
\State \hspace{1.2em} \textbf{For} each Haar $\in$ $H_t$
\State \hskip2.4em \textbf{Add noise}
\State \hskip2.4em $\widehat{H_t} \longleftarrow$($H_{t}$(Haar)+$\mathcal{N}$(0, $\dfrac{\sigma^2}{(W_{(Haar)})^2}$)
\State \hspace{1.2em} \textbf{Inverse wavelet transform}
\State \hspace{1.2em} $\widehat{w}_{t} \longleftarrow \dfrac{1}{L}(\Sigma_{i} H_{t}$)
\State \hspace{1.2em}\textbf{Return} $\widehat{w}_{t}$
  \end{algorithmic}
\end{algorithm}
\subsection{One-dimensional Haar Wavelet Transform}\label{sec::1d}

\begin{figure}[!t]
\centering
\includegraphics[width=3.3in]{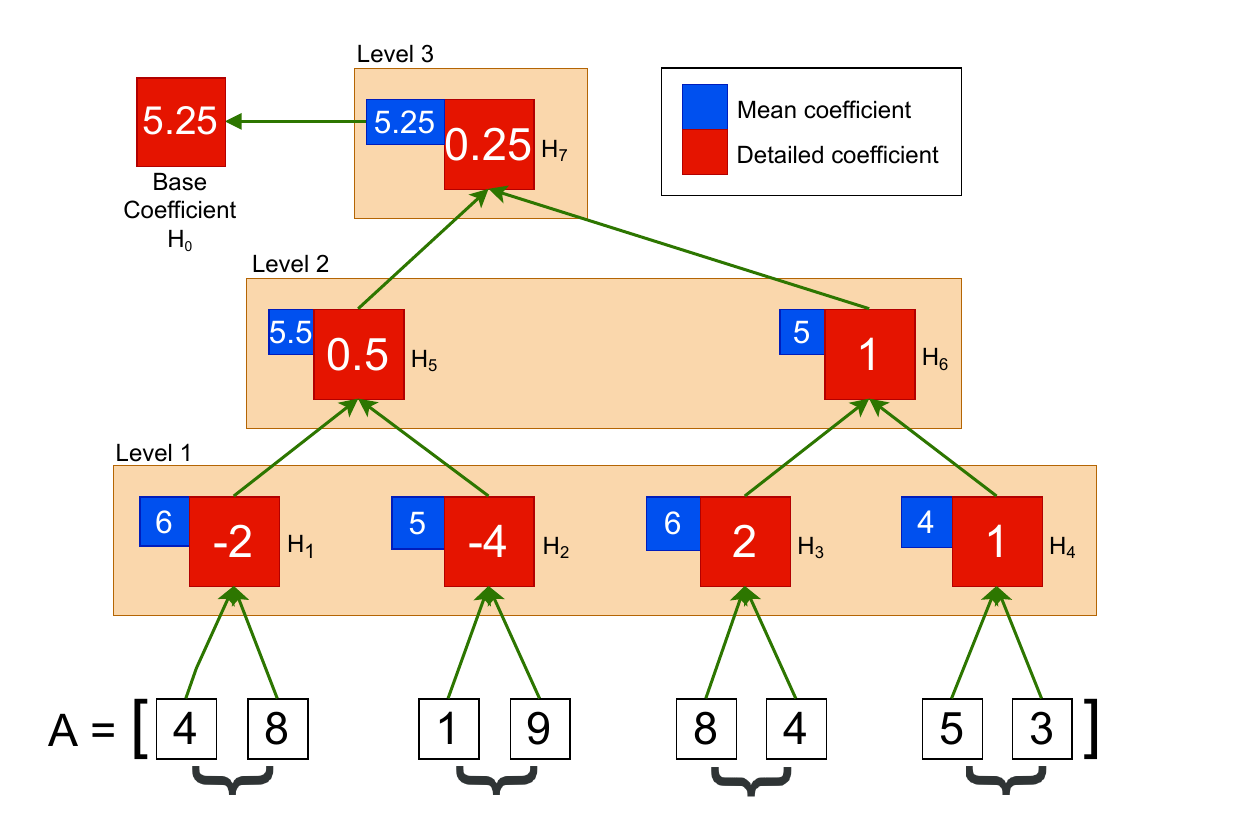}
\caption{Decomposition tree illustration of Haar wavelet transform.}
\label{tree}
\end{figure}

As an example, consider a data vector A with integer values such that A=[4, 8, 1, 9, 8, 4, 5, 3].
The HWT computation for this data vector is started by creating value pairs adjacent to one another. These pairs are then averaged to provide the mean coefficients [6, 5, 6, 4], which provide a new lower-level representation of the values. In other words, the first pair's [4, 8] mean is 6, the second pair's [1, 9] mean is 5, and so forth. We will compute the detailed coefficients for each of these pairings in parallel with the average computations and save them as the first-level wavelet coefficients. This is calculated using the formula $(L-R)/2$, where $L$ and $R$ are the first and second values in the pair, respectively. The results of the detailed coefficient computation for the example above are the coefficients [-2, -4, 2, 1]. Now, this procedure will be carried out $\log_2(m)$ times, where m is the array's cardinality, until there is only one detailed coefficient left. We only take into account the mean values for all calculations at each level of the transformation, saving the detailed coefficients separately. Table 2 provides a summary of the wavelet computations for the example mentioned above. As can be seen, this computation only works for arrays with $2^n (n \in \mathbb{Z}^+ )$ elements. As a result, any array that does not meet this requirement must be transformed into an array with two elements by zeros-padding. For instance, in order to apply an HWD to the array [1, 2, 3, 4, 5], zeros must be added to the end of the array, making it [1, 2, 3, 4, 5, 0, 0, 0].

\begin{table}[]
\centering
\caption{Haar wavelet transform computation.}
\label{hwt}
\begin{tabular}{lll}
\hline
 & & \\
Coefficient level & Mean coefficient      & Detailed coefficient \\ 
 & & \\ \hline
  & & \\
3                 & {[}4, 8, 1, 9, 8, 4, 5, 3{]} &                      \\ 
2                 & {[}6, 5, 6, 4{]}         & {[}-2, -4, 2, 1{]}      \\ 
1                 & {[}5.5, 5{]}           & {[}0.5, 1{]}         \\ 
0                 & 5.25                  & {[}0.25{]}           \\ \hline
\end{tabular}
\end{table}

In order to more easily refer to each computed coefficient, the wavelet coefficients can also be expressed in a decomposition tree, as shown in Fig.~ \ref{tree}. However, even though the mean coefficients are calculated in each step, they just serve as shadow coefficients, and only detailed coefficients are used for all calculations after the wavelet transform, including the backwards transformation. The base coefficient $H_0$, which is the average of the final level of computation, will also be used.

Invertibility is one of HWT's most intriguing characteristics. As shown below, each element in the original array may be recreated using the wavelet coefficients determined $[H_0, H_1, H_2, H_3, H_4, H_5, H_6, H_7]$ for our prior example.

\begin{equation}
\label{eq:600}
X=H_0 + \sum_{i=1}^{l}(S_{i}H_i).
\end{equation}

We must first take, the wavelet coefficients' decomposition tree illustration, into account in order to achieve a value for $S_i$. Using the decomposition tree illustration, a chain of wavelet coefficients connected to each element of the input array can be built. The ancestral wavelet coefficients chain of the third element (1) of the input array, for instance, are H2 H5 H7. Since the wavelet coefficients are generated in pairs, it is possible to determine whether the branch connecting the wavelet coefficient in the $l^{th}$ level and the wavelet coefficient in the $(l-1)^{th}$ level is on the left or right side. As a result, depending on whether the coefficient below is on the left or right side, $S_i$ equals 1 or -1.

As previously mentioned, our method adds independently generated Gaussian noise with a standard deviation of ${\sigma/W_{(Haar)}}$ to each wavelet coefficient, where $\sigma$ is the input parameter and $W_{(Haar)}$ is a weight function that changes depending on the coefficient's level in the decomposition tree. This noise is calibrated to the clipping threshold $C$. Here, $W_{(Haar)}$ is defined as follows for each coefficient: a base coefficient $W_{(Haar)}$ = m, where m is the size of the input array in terms of elements. $W_{(Haar)}$ = $2^l$ follows for any coefficient in level l. For instance, $W_{(Haar)}$ values for the computed Haar wavelet coefficients [$H_0, H_1, H_2, H_3, H_4, H_5, H_6, H_7$] shown in Fig.~\ref{tree} will be [8, 2, 2, 2, 2, 4, 4, 8].

\color{black}
In traditional DP learning techniques, sensitivity is managed through a clipping strategy that bounds the sensitivity of gradients using a fixed norm bound $C$. However, After applying the HWT to the gradients, the resulting wavelet coefficients at each level of the transform are clipped using a per-level norm bound. Each level of the HWT introduces different weight functions, which causes the sensitivity of the coefficients to vary across levels. This is a crucial difference from the flat clipping approach traditionally used in gradient-based methods. The weighted clipping across different levels of the wavelet transform enables us to refine the sensitivity control, as it adapts to the structure of the transformed data.

By leveraging the multi-level structure of the wavelet coefficients, the sensitivity reduction is achieved at each level, resulting in a lower overall noise variance when noise is injected. This reduction in noise variance significantly improves the utility of the model while preserving the same DP guarantees. The wavelet-based noise injection, combined with sensitivity-aware clipping, results in an enhanced privacy-utility tradeoff compared to standard approaches that apply a flat clipping strategy directly to the gradients.

\color{black}

\section{Theoretical analysis}\label{sec:ta}

This section provides theoretical justification for the claims we are making in this research. starting with noise variance constraints, then DP analysis, and finally a convergence performance analysis.

\subsection{Noise Variance Bounds Analysis}
Assume that a vanilla DP-SGD or DP-FedAvg algorithm using any accountant technique, such as the moments accountant [15] or the Renyi Differential Privacy (RDP) accountant [33], provides $(\varepsilon, \delta)$ privacy guarantee when a noise with a variance of $\sigma^2$ is introduced to gradients. The noise variance bounds for our approach are derived from the following lemma:                                              
\vspace{1em}

\begin{lemma}
\color{black}Assume that we perform HWT on a set of gradients $(G)$ with $m$ elements, 
which results in a set of wavelet coefficients $(H)$ where we add independent noise with standard deviation of $\dfrac{\sigma_{(Haar)}}{W_{(Haar)}}$. 
The set of noisy gradients ($G^*$) reconstructed from the inverse wavelet transform has a noise variance of $\dfrac{2+\log_2 (m)}{2} (\sigma_{(Haar)})^2$.

\end{lemma}

\begin{proof}
Consider the wavelet transformation decomposition tree $T$ of a set of gradients $G$. We have a decomposition tree of noisy wavelet coefficients ($T^*$) after adding independent noise to each wavelet coefficient. Recollect that we used equation \ref{eq:600} to reproduce the noisy gradients from $T^*$. Each element in the set of noisy gradients can be expressed as a weighted sum of base coefficients ($H_0$) and all coefficients in its ancestral chain in the decomposition tree. The base coefficient ($H_0$) has a weight of 1, while the other coefficients have weights of 1 or -1 depending on whether the leaves in the level below that are connected to the wavelet coefficient are on the left or right side. As a result, the weighted sum ($S$) for the noisy gradients ($G^*$) can be expressed as follows:

\begin{equation}
S = mH_0 + \sum_{H \in T^* \setminus H_0}(H (L-R)).
\end{equation}

\begin{figure}[b!]
\centering
\includegraphics[width=3.3in]{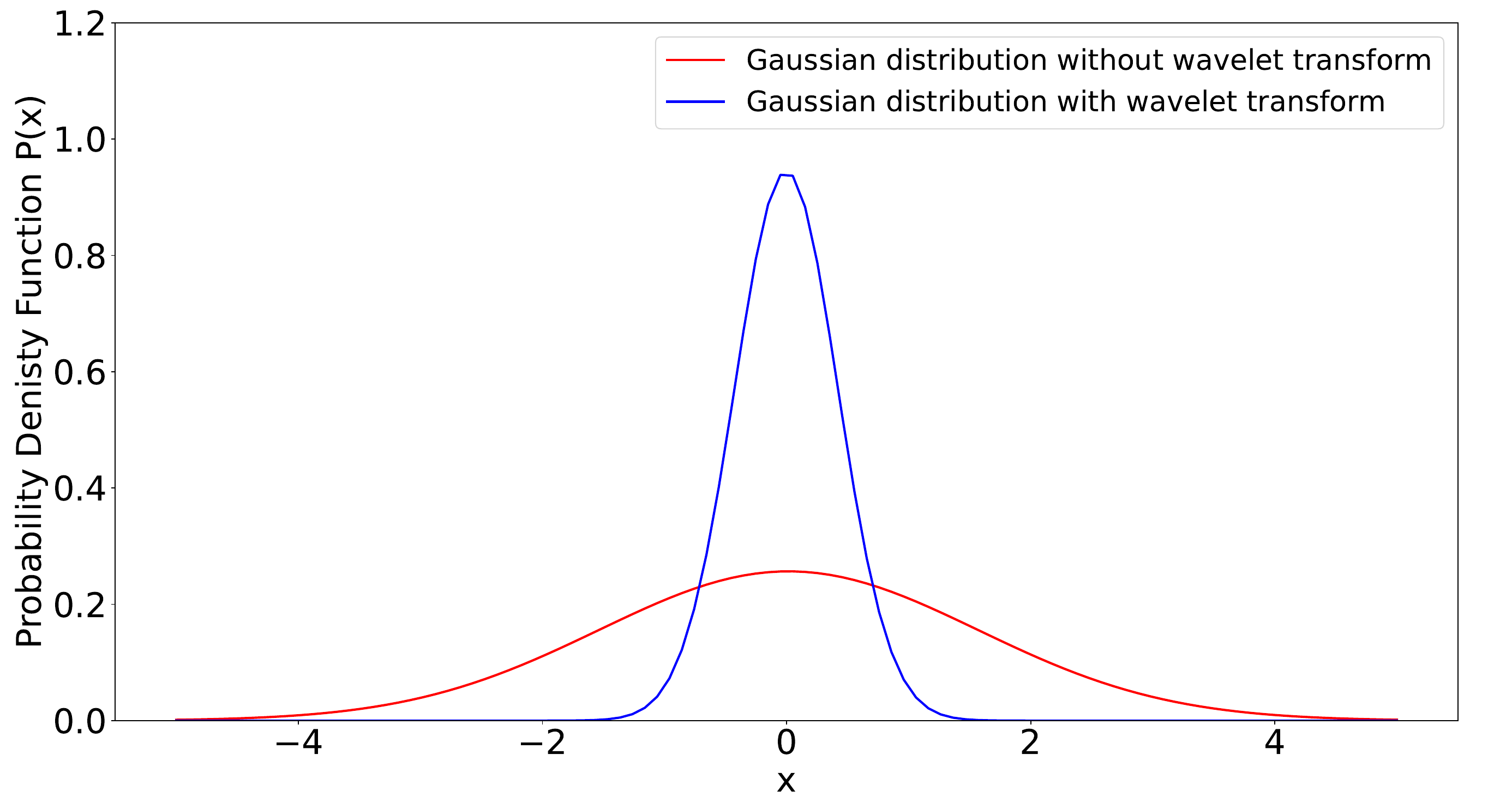}
\caption{Gaussian distribution curves with varying noise variance used in DP training. The red curve represents the original noise variance, while the blue curve represents the reduced noise variance proposed in our method. The x-axis shows the range of noise values, which is a real-valued axis that represents all possible noise values that can be added to the training process, and the y-axis represents the probability density.}
\label{gauss}
\end{figure}

Where L and R denote the number of leaves in the wavelet coefficient H's left and right side subtrees. When considering a single element $X \in G^*$, the maximum number of leaves in the left or right subtree of any coefficient is $2^{(level(H)-1)}$, where $level(H)$ denotes the level of the coefficient H. We now have $L, H \in [0, 2^{(level(H)-1)}]$ As a result, the base coefficient is $\mid(L-R)\mid \leq 2^{(level(H)-1)} $ as $W_{(Haar)}$. The noise variance contributed to $G^*$ by the base coefficient is as follows:
\begin{equation}
m^2.\dfrac{\sigma_{(Haar)}^2}{n^2} = \sigma_{(Haar)}^2.
\end{equation}
Also recall, for every other coefficient $H$, $W_{(Haar)}$ = $2^{level(H)}$. Therefore the noise variance contributed by a single coefficient to $G^*$ is
\begin{equation}
(L-R)^2\dfrac{\sigma_{(Haar)}^2}{(2^{level(H)})^2} = {(2^{(level(H)-1)})}^2\dfrac{\sigma_{(Haar)}^2}{2^{2level(H)}}=\dfrac{\sigma_{(Haar)}^2}{4}.
\end{equation}

Since the number of levels $l=(1+\log_2(m))$ in wavelet transforms, The maximum noise variance at $G^*$ is, 

\begin{equation}
\begin{aligned}
\sigma_{(Haar)}^2 +2l\dfrac{\sigma_{(Haar)}^2}{4} &=\sigma_{(Haar)}^2 +2(1+\log_2(m))\dfrac{\sigma_{(Haar)}^2}{4} \\ &= \dfrac{2+\log_2 (m)}{2} (\sigma_{(Haar)})^2.
\end{aligned}
\end{equation}
\end{proof}

Fig.~\ref{gauss} depicts the empirical probability density function of noise in two methods for a fixed $(\varepsilon, \delta)$, at the noisy gradients. When compared to our method, the standard noise injection scheme's plot is stretched, indicating that the predictions made from these data are far from the actual results, demonstrating how our method provides significantly improved utility in the final model.

\subsection{Differential Privacy Analysis}

\color{black}
In this work, we utilize two well-established privacy accounting techniques, namely the moments accountant \cite{abadi2016deep} and the analytical moments accountant using RDP \cite{wang2019subsampled}, to provide our $(\varepsilon, \delta)$ differential privacy guarantees. We want to clarify that our contribution does not lie in introducing a new privacy loss accounting method but in demonstrating how these techniques can be effectively applied in conjunction with the HWT to achieve a better utility-privacy tradeoff. By leveraging the wavelet-based noise injection mechanism, we are able to reduce the overall noise variance while maintaining the same privacy guarantees as the traditional DP-SGD and DP-FedAvg algorithms.
\color{black}

\begin{lemma}

When moments accountant technique is used, our method guarantees ($\varepsilon, \delta$) differential privacy for any  $\varepsilon<c_1q_c^2T$ and $\delta>0$, while saving a factor of $\sqrt{(2+\log_2 (m))/2}$ in the asymptotic bound from the vanilla DP-SGD algorithm,  if $\sigma_{(Haar)}$ is chosen as, 
\begin{center}
$\sigma_{(Haar)}=c_2\dfrac{K}{\varepsilon}\sqrt{\dfrac{2T\log(1/q_c)}{2+\log_2 (m)}}$, 
\end{center}
where $c_1$ and $c_2$ are constants, the number of steps $T$, sampling probability $q_c=L_n/N$, where $L_n$ and $N$ are the lot size and the size of the input dataset respectively.
\end{lemma}

\begin{proof}
According to \cite{abadi2016deep}, for any $\varepsilon<c_1q_c^2T$ and $\delta>0$, the vanilla DP-SGD algorithm is ($\varepsilon, \delta$) differentially private, and if $\sigma$ is, 
\begin{equation}
    \sigma=c_2\dfrac{q_c\sqrt{T\log(1/q_c)}}{\varepsilon}.
\end{equation}
Lemma 1 establishes that $\sigma=(\sigma_{(Haar)})\sqrt{(2+\log_2 (m))/2}$. As a result, for any $\varepsilon<c_1q_c^2T$ and $\delta>0$, our method is ($\varepsilon, \delta$) differentially private iff, 

\begin{equation}
\begin{split}
(\sigma_{(Haar)})\sqrt{(2+\log_2 (m))/2}=c_2\dfrac{q_c\sqrt{T\log(1/q_c)}}{\varepsilon}, \\
\sigma_{(Haar)}=c_2\dfrac{K}{\varepsilon}\sqrt{\dfrac{2T\log(1/q_c)}{2+\log_2 (m)}},
\end{split}
\end{equation}
since $n\geq1$, it can be observed that $\sqrt{2/(2+\log_2 (m))}\leq$1. At the asymptotic bound, our method saves a factor of $\sqrt{(2+\log_2 (m))/2}$.
\end{proof}

As our modifications are applied prior to the accountant, the results presented above are also applicable to the analytical moments accountants who use RDP \cite{wang2019subsampled}in a ``black-box" manner. Suppose we have a standard mechanism that provides ($\varepsilon, \delta$) differential privacy with a noise variance of $\sigma^2$. To achieve the same ($\varepsilon, \delta$) privacy guarantee, the noise variance at noisy gradients in our method will be $\dfrac{2+\log_2 (m)}{2} (\sigma_{(Haar)})^2=\sigma^2$, saving a factor of $\sqrt{(2+\log_2 (m))/2}$ in the asymptotic bound.
\subsection{Convergence Analysis of DP-SGD-WAV}
\color{black}
In this section, we will examine the proposed algorithm's convergence performance. We accomplish this by analyzing the expected difference in the loss function with Gaussian noise between adjacent aggregations. First, we make the following \textbf{assumptions:}
\begin{enumerate} 
    \item Gradient dissimilarity is bounded: There exists constant $B_1$ and $B_2$ so that $\sum_{k=1}^{K}\{||\triangledown L_k(\theta) - \triangledown L(\theta) ||^2\} \le  B_1||\triangledown L(\theta)||^2 + \dfrac{B_2^2}{K}$.
    \item $L_k(\theta)$ is Lipschitz continuous; i.e. there exists a real constant M so that $||\triangledown L_k(\theta)-\triangledown L_k(\theta')||\le M||\theta-\theta'||$ holds true.
    \item L($\theta$) satisfies Polyak-Lojasiewicz inequality if there exists a positive scalar $\mu>0$ such that $\dfrac{1}{2}||\triangledown L(\theta)||^2 \geq \mu(L(\theta)-L(\theta^*))$. where, $\theta^*$ is the minimizer of $L(\theta)$.

\end{enumerate}

\begin{theorem}
\label{lem_con}
Let $\{\theta_t\}_{t \geq 0}$ be the sequence of model parameters generated by the DP-FedAvg via Wavelets algorithm (Algorithm \ref{dp-fed}). Assume that the local loss functions $L_k(\theta)$ are Lipschitz continuous and the gradient dissimilarity is bounded. Then, the expected difference between the global loss function $L(\theta)$ at iteration $t+1$ and the optimal loss $L(\theta^*)$ is bounded as follows:
\begin{equation}
\label{eq9}
    \begin{aligned}
        \mathbb{E}[L(\theta_{t+1})] - L(\theta^*) \leq \triangle_t\mathbb{E}[L(\theta_t)-L(\theta^*)]+c_t\\+\dfrac{\eta_t}{2}[-1+\lambda M\eta_t(\dfrac{B_1+K}{K})]||\sum^K_{k=1}\dfrac{d_k}{q_cd}\nabla\theta_{t}^k\widehat{L}_k(\theta_{t}, x_{i}^k)||^2 \\+ B_t\sum^{t-1}_{j=t_c+1}\eta_j^2||\sum^K_{k=1}\dfrac{d_k}{q_cd}\nabla\theta_{t}^k\widehat{L}_k(\theta_{t}, x_{i}^k)||^2.
    \end{aligned}
\end{equation}

where,

\begin{equation}
    \triangle_t = 1-\mu\eta_t,
\end{equation}
\begin{equation}
    c_t=\dfrac{\eta_tMB_2^2}{K}[\dfrac{\eta_t}{2}+\dfrac{M(K+1)}{K}\sum^{t-1}_{j=t_c+1}\eta_j^2,
\end{equation}
\begin{equation}
    B_t = \dfrac{\lambda(K+1)\eta_tM^2}{K^2}(B_1+N).
\end{equation}
\end{theorem}

\begin{proof}
See Appendix~\ref{sec:app1}
\end{proof}
Using theorem~\ref{lem_con}, we can describe the convergence characteristics of the DP-FedAvg via Wavelets algorithm. The lemma provides a bound on the expected difference between the global loss function $L(\theta)$ at iteration $t+1$ and the optimal loss $L(\theta^*)$. This bound is composed of several terms that reflect different aspects of the algorithm's behavior. The term $\triangle_t = 1 - \mu \eta_t$ ensures that as long as $0 < \mu \eta_t < 1$, the algorithm will decrease the error term $\mathbb{E}[L(\theta_t) - L(\theta^*)]$ in expectation, driving the expected loss closer to the optimal loss with each iteration. The term $c_t$ accounts for the influence of noise introduced for differential privacy and the gradient dissimilarity. As the learning rate $\eta_t$ decreases, $c_t$ also diminishes, indicating that the impact of these factors becomes less significant over time.

Additionally, the term 

$\frac{\eta_t}{2} \left[-1 + \lambda M \eta_t \left(\frac{B_1 + K}{K}\right)\right] \left\| \sum_{k=1}^K \frac{d_k}{q_cd} \nabla_{\theta_t^k} \widehat{L}_k(\theta_t, x_i^k) \right\|^2$,

highlights the role of the weighted gradients in the convergence process. If the condition $\lambda M \eta_t \left(\frac{B_1 + K}{K}\right) < 1$ holds, this term contributes negatively, further promoting convergence by reducing the loss. Lastly, the term $B_t \sum_{j=t_c+1}^{t-1} \eta_j^2 \left\| \sum_{k=1}^K \frac{d_k}{q_cd} \nabla_{\theta_t^k} \widehat{L}_k(\theta_t, x_i^k) \right\|^2$ reflects the accumulated effect of past gradients, emphasizing the importance of a proper decay schedule for $\eta_t$ to ensure effective convergence. These components collectively illustrate that the convergence rate of the DP-FedAvg via Wavelets algorithm is influenced by the learning rate, gradient dissimilarity, differential privacy noise, and the weighted contribution of local gradients. By tuning these parameters appropriately, the algorithm can achieve efficient convergence to the optimal solution.

\color{black}

\begin{figure*}[!h]
\centering
\includegraphics[width=\textwidth]{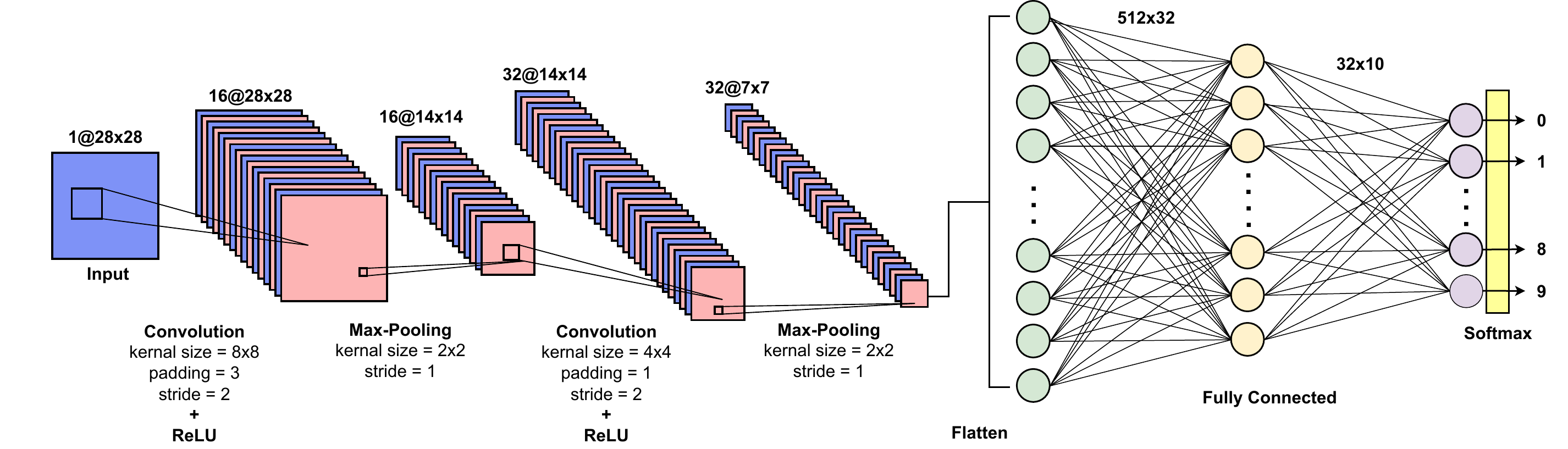}
\caption{Architecture of the Convolutional Neural Network (CNN) used to train on MNIST and Fashion MNIST datasets, consisting of multiple convolutional layers followed by pooling layers and fully connected layers.}
\label{mnist_cnn}
\end{figure*}

\begin{figure*}[!h]
\centering
\includegraphics[width=\textwidth]{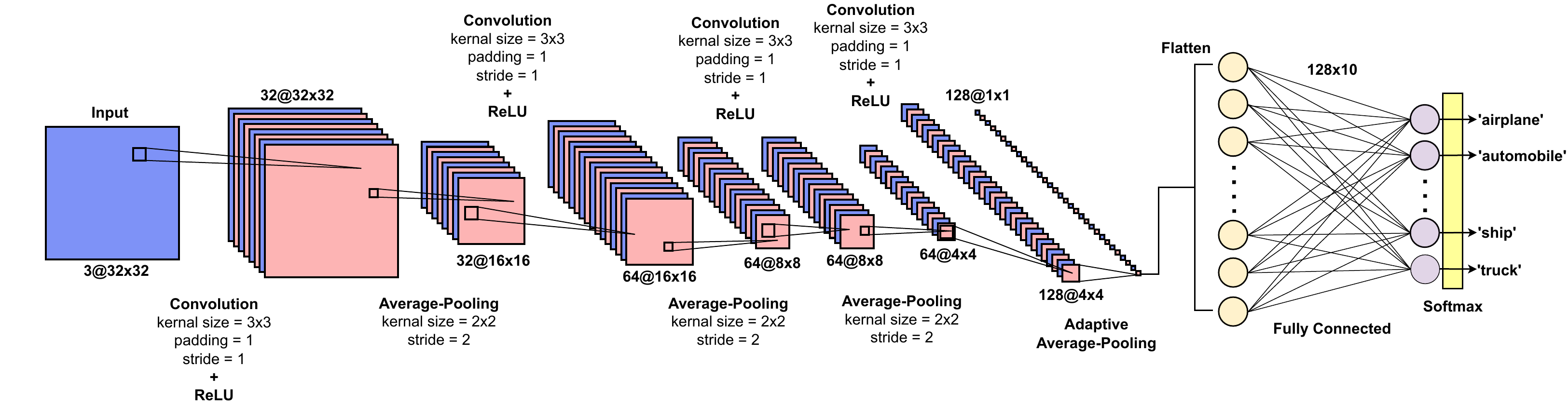}
\caption{Architecture of the Convolutional Neural Network (CNN) used to train on CIFAR10 dataset, consisting of multiple convolutional layers followed by pooling layers and fully connected layers.}
\label{cifar_cnn}
\end{figure*}

From our theoretical analysis, it is obvious that artificial noise with a high variance ($\sigma^2$) may improve the DP's privacy protection performance. However, based on the RHS of (\ref{eq9}), a large $\sigma^2$ may increase the expected difference of the loss function between two consecutive aggregations, causing convergence performance to deteriorate. In order to minimize this deterioration, our proposed algorithm DP-FedAvg-WAV saves a factor of $\dfrac{2+\log_2 (m)}{2} (\sigma_{(Haar)})^2$ in the noise variance added to the model gradients compared to its vanilla algorithm, resulting in superior convergence performance. It is worth noting that when the number of clients $K=1$, DP-FedAvg becomes identical to DP-SGD. As a result, it is reasonable to conclude that these results are also valid for DP-SGD-WAV. We visualize these results via experiments in Section~\ref{sec:ex}.

\begin{figure*}[!t]
\centering
\includegraphics[trim=180 40 170 50,clip,width=\textwidth]{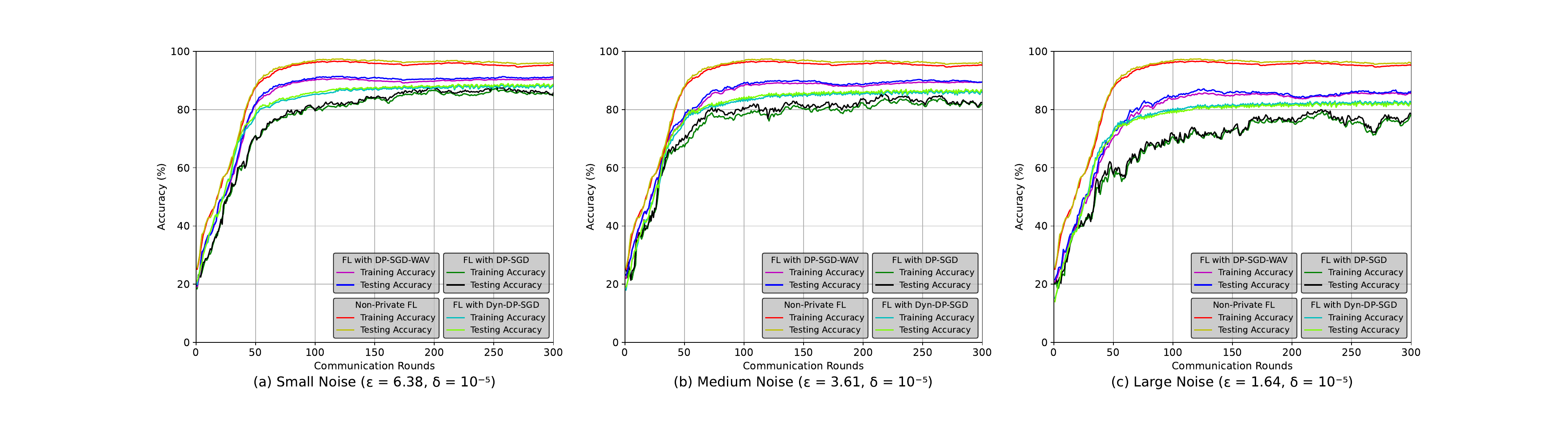}
\caption{The learning accuracies achieved through FL by applying sample-level DP guarantees to local SGD training on the MNIST dataset. DP noise was added under three different privacy settings.}
\label{mnistsgd}
\end{figure*}
\begin{figure*}[!t]
\centering
\includegraphics[trim=180 40 170 50,clip,width=\textwidth]{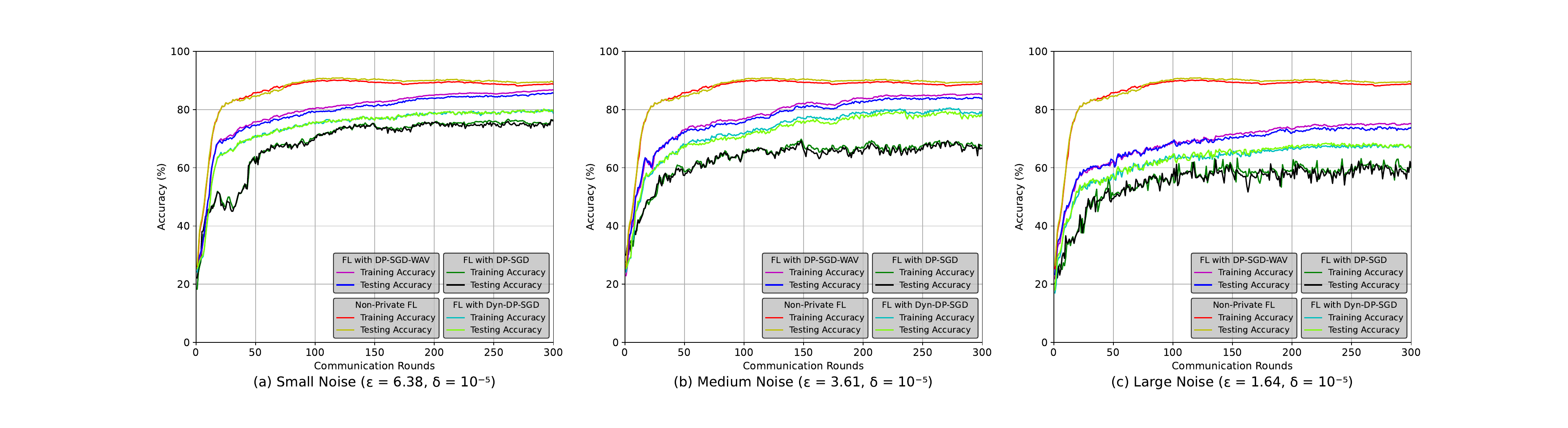}
\caption{The learning accuracies achieved through FL by applying sample-level DP guarantees to local SGD training on the Fashion MNIST dataset. DP noise was added under three different privacy settings.}
\label{fmnistsgd}
\end{figure*}
\begin{figure*}[!t]
\centering
\includegraphics[trim=180 40 170 50,clip,width=\textwidth]{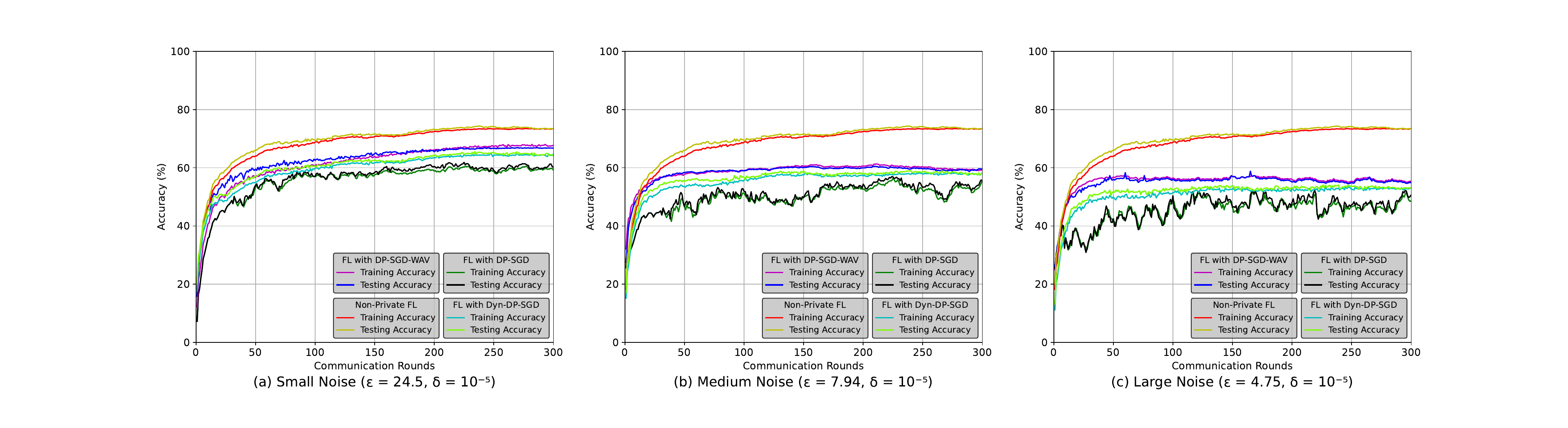}
\caption{The learning accuracies achieved through FL by applying sample-level DP guarantees to local SGD training on the CIFAR10 dataset. DP noise was added under three different privacy settings.}
\label{cifarsgd}
\end{figure*}

\section{Experiments}\label{sec:ex}
In this section, we show that our approach outperforms the state-of-the-art DP algorithms currently available. We do a number of tests on various datasets with various privacy budgets in order to achieve this.
\subsection{Experimental Setup}
Our experiments were conducted in Python, using the PyTorch framework on a system equipped with an Intel Core i7 11th generation processor and NVIDIA RTX 3080 GPU acceleration. To track cumulative privacy loss, we used the RDP accountant from Google's differential privacy libraries. For the MNIST dataset and the Fashion-MNIST dataset, each consisting of 70,000 grayscale images, we employed a convolutional neural network (CNN) model, whose architecture is illustrated in Fig.~\ref{mnist_cnn}. The MNIST dataset contains handwritten digits, while the Fashion-MNIST dataset contains images of clothing items. Both datasets are split into 60,000 training samples and 10,000 test samples, with each image having a resolution of 28x28 pixels. The CIFAR10 dataset, which consists of 60,000 colour images of size 32x32, divided into ten classes, was used to evaluate the performance of our proposed method on larger and more complex images. For this dataset, we used a much deeper CNN model, as outlined in Fig.~\ref{cifar_cnn}. Overall, the datasets were chosen to evaluate our proposed method's effectiveness in different image recognition tasks.

\begin{figure*}[!t]
\centering
\includegraphics[trim=180 40 170 50,clip,width=\textwidth]{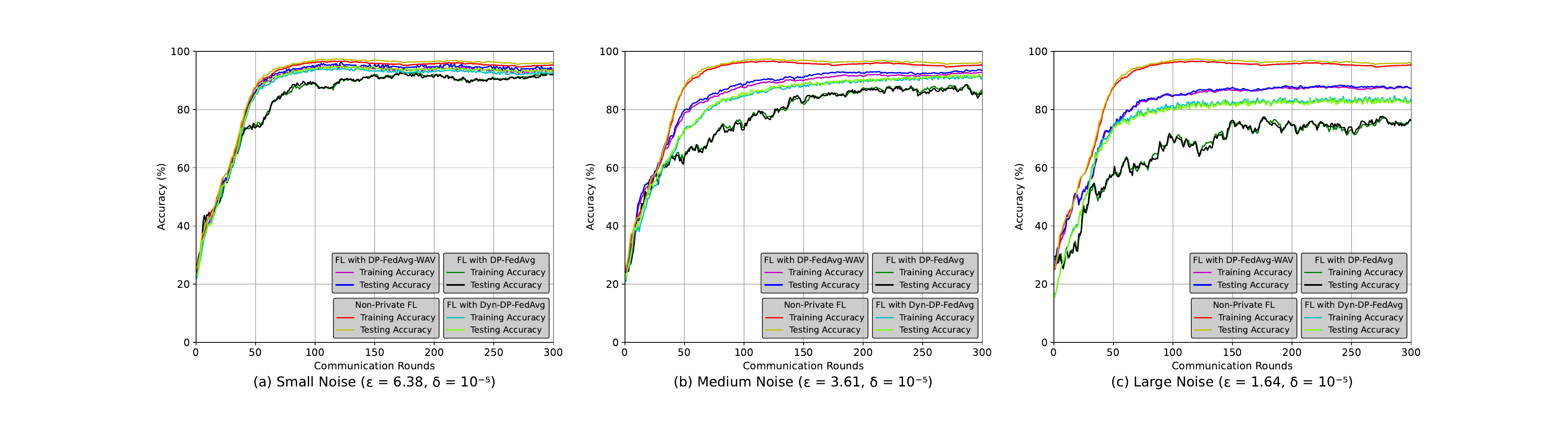}
\caption{The learning accuracies achieved through FL by applying user-level DP guarantees to the aggregation step (FedAvg algorithm) on the MNIST dataset. DP noise was added under three different privacy settings.}
\label{mnistfed}
\end{figure*}
\begin{figure*}[!t]
\centering
\includegraphics[trim=180 40 170 50,clip,width=\textwidth]{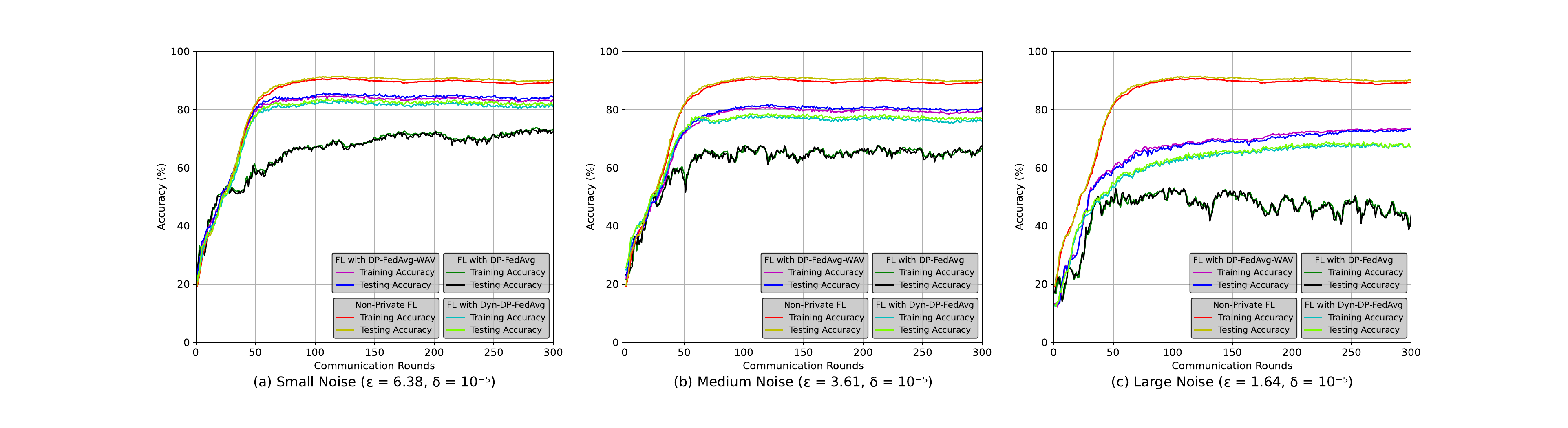}
\caption{The learning accuracies achieved through FL by applying user-level DP guarantees to the aggregation step (FedAvg algorithm) on the Fashion MNIST dataset. DP noise was added under three different privacy settings.}
\label{fmnistfed}
\end{figure*}
\begin{figure*}[!t]
\centering
\includegraphics[trim=180 40 170 50,clip,width=\textwidth]{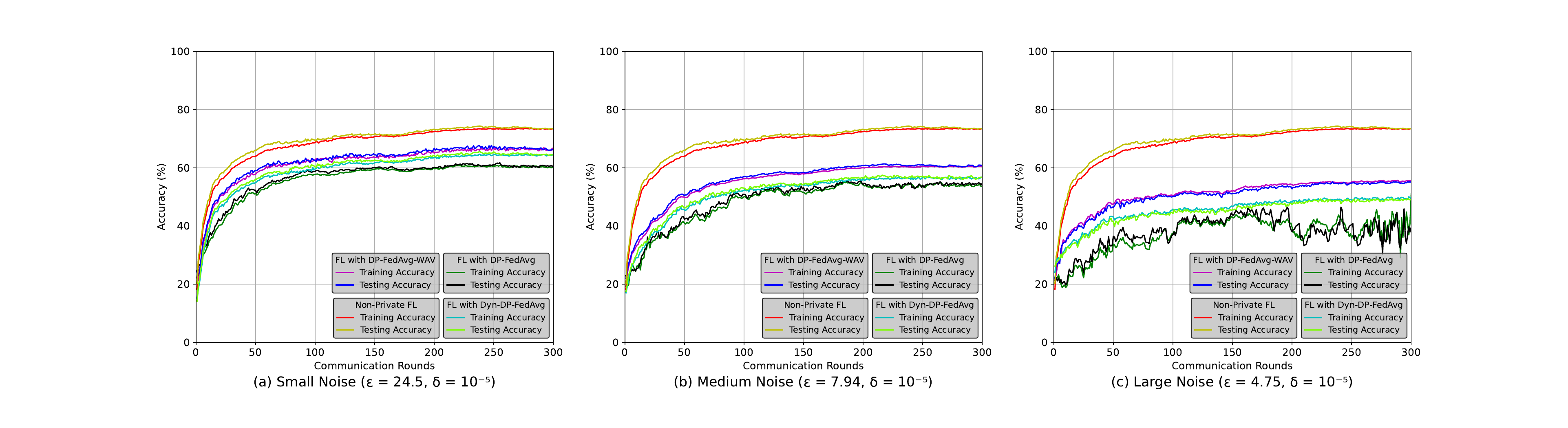}
\caption{The learning accuracies achieved through FL by applying user-level DP guarantees to the aggregation step (FedAvg algorithm) on the CIFAR10 dataset. DP noise was added under three different privacy settings.}
\label{cifarfed}
\end{figure*}

All hyperparameters used in these experiments were selected through a rigorous hyperparameter tuning process. We tested a wide range of values for each hyperparameter, and the best possible hyperparameters were selected based on their impact on the model's performance, privacy guarantees, and computational efficiency. We tuned various hyperparameters, including the learning rate, batch size, weight decay, number of layers, number of filters, and kernel size. Additionally, we experimented with different values of the privacy parameter epsilon to ensure sufficient privacy protection while maintaining acceptable model accuracy. The hyperparameter tuning process was conducted using a combination of manual and automated methods. For the manual tuning, we employed our domain knowledge and prior experience to select the hyperparameters. For the automated tuning, we used a popular tool, Grid Search.

\color{black}In our implementation, the median clipping norm is calculated only during the initial rounds of training using a public proxy dataset. Since the client data in our FL framework is i.i.d., this ensures that the clipping threshold is representative of the client data distribution. The median clipping value is then applied consistently throughout the training process, providing a fixed sensitivity bound.\color{black} The noise scales $\sigma$ were set to 1, 1.5, and 3 for small, medium, and large noise settings, respectively. The resulting $\epsilon$ values are recorded in Table~\ref{tab:opti} with $\delta$ set to $10^{-5}$.\color{black}
The FL process was conducted using 100 clients, where the datasets were divided equally among them. \color{black} The training was carried out for 300 communication rounds, with each client performing 50 local training epochs per round, and the global model was evaluated after every epoch. \color{black}

\subsection{Baselines}

To evaluate the effectiveness of our proposed method, we compared it against three baselines, which are widely used in the literature for federated learning with differential privacy. The first baseline is non-private federated learning \cite{mcmahan2017communication}. In this approach, the clients send their model updates to the central server, which aggregates the updates and broadcasts the new model to all clients. This baseline serves as a performance benchmark for our proposed method. In experiments conducting sample-level privacy for federated learning settings, we used DP-SGD as the baseline for differential privacy. In this approach, differential privacy is introduced at the sample level by adding noise to the gradients. On the other hand, in experiments conducting client-level privacy for federated learning settings, we used DP-FedAvg as the baseline for differential privacy. In this approach, differential privacy is introduced at the client level by adding noise to the model updates. The third baseline is federated learning with dynamic DP-SGD (Dyn-DP-SGD) or Dynamic FedAvg (Dyn-DP-FedAvg) \cite{du2021dynamic}, depending on the experiment setting. In the sample-level privacy setting, we use Dyn-DP-SGD, where the noise added to the gradients is adapted dynamically based on the sensitivity of the data. In the client-level privacy setting, we use Dyn-DP-FedAvg, where the noise added in the aggregation is adapted based on the number of participating clients. 

\subsection{FL with Sample-Level DP}

\begin{table*}[t!]
\centering
\caption{Accuracy improvement our method provides when DP-Adam and DP-AdaGrad optimizers are used in place of DP-SGD, Here $\delta=10^{-5}$.}
\label{tab:opti}
\begin{tabular}{|c|ccc|ccc|ccc|}
\hline
                                         & \multicolumn{3}{c|}{\cellcolor[HTML]{FFFFFF}{\color[HTML]{000000} MNIST}}                                                                           & \multicolumn{3}{c|}{\cellcolor[HTML]{FFFFFF}{\color[HTML]{000000} Fashion MNIST}}                                                                   & \multicolumn{3}{c|}{\cellcolor[HTML]{FFFFFF}{\color[HTML]{000000} CIFAR10}}                                                                         \\ \cline{2-10} 
                                         & \multicolumn{3}{c|}{\cellcolor[HTML]{C0C0C0}DP property}                                                                                & \multicolumn{3}{c|}{\cellcolor[HTML]{C0C0C0}DP property}                                                                                & \multicolumn{3}{c|}{\cellcolor[HTML]{C0C0C0}DP property}                                                                                \\ \cline{2-10} 
\multirow{-3}{*}{\begin{tabular}[c]{@{}c@{}}Optimization \\ Algorithm\end{tabular}} & \multicolumn{1}{c|}{\cellcolor[HTML]{C0C0C0}$\varepsilon=6.38$} & \multicolumn{1}{c|}{\cellcolor[HTML]{C0C0C0}$\varepsilon=3.61$} & \cellcolor[HTML]{C0C0C0}$\varepsilon=1.64$ & \multicolumn{1}{c|}{\cellcolor[HTML]{C0C0C0}$\varepsilon=6.38$} & \multicolumn{1}{c|}{\cellcolor[HTML]{C0C0C0}$\varepsilon=3.61$} & \cellcolor[HTML]{C0C0C0}$\varepsilon=1.64$ & \multicolumn{1}{c|}{\cellcolor[HTML]{C0C0C0}$\varepsilon=24.5$} & \multicolumn{1}{c|}{\cellcolor[HTML]{C0C0C0}$\varepsilon=7.94$} & \cellcolor[HTML]{C0C0C0}$\varepsilon=4.75$ \\ \hline
DP-Adam                                  & \multicolumn{1}{c|}{3.56\%}                           & \multicolumn{1}{c|}{4.9\%}                             & 6.12\%                             & \multicolumn{1}{c|}{8.22\%}                           & \multicolumn{1}{c|}{10.23\%}                           & 12.97\%                            & \multicolumn{1}{c|}{6.01\%}                           & \multicolumn{1}{c|}{6.83\%}                            & 7.93\%                             \\ \hline
DP-AdaGrad                               & \multicolumn{1}{c|}{2.31\%}                           & \multicolumn{1}{c|}{3.29\%}                            & 4.97\%                             & \multicolumn{1}{c|}{6.61\%}                           & \multicolumn{1}{c|}{8.19\%}                            & 10.88\%                            & \multicolumn{1}{c|}{4.25\%}                           & \multicolumn{1}{c|}{5.53\%}                            & 7.03\%                             \\ \hline
\end{tabular}
\end{table*}

The learning accuracies achieved through FL with sample-level DP on three different datasets, MNIST, Fashion MNIST, and CIFAR-10, are presented in Figures \ref{mnistsgd}, \ref{fmnistsgd}, and \ref{cifarsgd} respectively. The figures show the performance of our proposed method, FL with DP-SGD-WAV, compared to three baseline approaches: vanilla DP-SGD \cite{abadi2016deep}, dynamic DP-SGD \cite{du2021dynamic}, and non-private federated learning \cite{mcmahan2017communication}. The non-private FL serves as a performance benchmark in our evaluation. We observe that our proposed method outperforms the state-of-the-art approaches across all three noise settings, small, medium, and large, demonstrating its effectiveness in improving the utility of DP training.

In the MNIST dataset, we observed that our proposed method outperformed the three baseline approaches across all three noise settings. Specifically, in the small noise setting, our proposed method achieves a 4.9\% increase in accuracy compared to vanilla DP-SGD, and a 2.87\% increase compared to dynamic DP-SGD. In the medium noise setting, we observe a similar trend with a 7.03\% and 4.2\% increase in accuracy compared to vanilla DP-SGD and dynamic DP-SGD, respectively. In the large noise setting, our proposed method again outperforms the baselines with a 9.16\% increase in accuracy compared to vanilla DP-SGD and a 4.82\% increase compared to dynamic DP-SGD.

For the Fashion MNIST dataset, our proposed method also outperformed the three baseline approaches across all three noise settings. In the small noise setting, our proposed method achieves a 9.93\% increase in accuracy compared to vanilla DP-SGD, and a 5.64\% increase compared to dynamic DP-SGD. In the medium noise setting, we observe a similar trend with a 12.7\% and 5.53\% increase in accuracy compared to vanilla DP-SGD and dynamic DP-SGD, respectively. In the large noise setting, our proposed method again outperforms the baselines with a 14.6\% increase in accuracy compared to vanilla DP-SGD and a 5.35\% increase compared to dynamic DP-SGD.

Finally, for the CIFAR-10 dataset, our proposed method also outperformed the three baseline approaches across all three noise settings. In the small noise setting, our proposed method achieves a 6.38\% increase in accuracy compared to vanilla DP-SGD, and a 1.871\% increase compared to dynamic DP-SGD. In the medium noise setting, we observe a similar trend with a 7.2\% and 1.88\% increase in accuracy compared to vanilla DP-SGD and dynamic DP-SGD, respectively. In the large noise setting, our proposed method again outperforms the baselines with a 9.1\% increase in accuracy compared to vanilla DP-SGD and a 2.05\% increase compared to dynamic DP-SGD.

These results demonstrate that our proposed method, FL with DP-SGD-WAV, achieves improved accuracy compared to existing sample-level DP methods in FL. The wavelet-based noise added to the gradients in our proposed method provides higher accuracies while maintaining privacy guarantees. This is congruent with our theoretical proof, in which we demonstrated that using the wavelet transform reduces the noise variance bounds by a factor of $\dfrac{2+\log_2 (m)}{2} (\sigma_{(Haar)})^2$. 

In addition to our experimental results, 
we extend our work to other optimizers such as DP-Adam \cite{kingma2014adam} and DP-AdaGrad \cite{duchi2011adaptive} to prove that the wavelet technique works with any optimizer in a black-box way to provide sample-level DP guarantees for FL applications. Similarly to the previous experiments, we run these in the MNIST, Fashion MNIST and CIFAR10 datasets with three different noise scales, with $\delta$ fixed to $10^{-5}$ and $\varepsilon$ varied to have varying values.
Table \ref{tab:opti} summarizes the accuracy increase on these two optimizers for large, medium  and small noise settings. It can be observed, the reduction of the noise variance that the HWT causes, by a factor of $\dfrac{2+\log_2 (m)}{2} (\sigma_{(Haar)})^2$ boosts the accuracies of these optimization techniques in a manner similar to the DP-SGD algorithm.

\subsection{FL with Client-Level DP}

We evaluated the effectiveness of our proposed DP-FedAvg-WAV, in improving learning accuracy while maintaining privacy guarantees in client-level FL settings. We conducted experiments on three datasets, MNIST, Fashion MNIST, and CIFAR-10, and compared our method to three baseline approaches: vanilla DP-FedAvg\cite{mcmahan2017learning}, Dyn-DP-FedAvg~\cite{du2021dynamic}, and non-private FL. Similar to the sample-level experiments the non-private FL served as a performance benchmark in our evaluation.

Fig.~\ref{mnistfed} shows our method achieved a 3.64\% increase in accuracy compared to vanilla DP-FedAvg and a 1.21\% increase compared to Dyn-DP-FedAvg in the small noise setting for the MNIST dataset. In the medium noise setting, we observed a similar trend with a 6.14\% and 1.82\% increase in accuracy compared to vanilla DP-FedAvg and Dyn-DP-FedAvg, respectively. In the large noise setting, our method again outperformed the baselines with a 12.97\% increase in accuracy compared to vanilla DP-FedAvg and a 4.93\% increase compared to Dyn-DP-FedAvg. Similar improvements were observed in Fig.~\ref{fmnistfed} for the Fashion MNIST dataset. Our method achieved a 13.2\% increase in accuracy compared to vanilla DP-FedAvg and a 2.4\% increase compared to Dyn-DP-FedAvg in the small noise setting. In the medium noise setting, we observed a 14.82\% and 3.08\% increase in accuracy compared to vanilla DP-FedAvg and Dyn-DP-FedAvg, respectively. In the large noise setting, our method shows a 27.3\% increase in accuracy compared to vanilla DP-FedAvg and a 4.39\% increase compared to Dyn-DP-FedAvg. Fig.~\ref{cifarfed} shows the learning accuracies for FL with client-level DP, trained with the CIFAR10 dataset. similar to the previous experiments our method achieved a 5.96\% increase in accuracy compared to vanilla DP-FedAvg and a 1.987\% increase compared to Dyn-DP-FedAvg in the small noise setting. In the medium noise setting, a 6.67\% and 4.02\% increase in accuracy compared to vanilla DP-FedAvg and Dyn-DP-FedAvg, respectively was observed. In the large noise setting, our method shows a 16.33\% increase in accuracy compared to vanilla DP-FedAvg and a 5.94\% increase compared to Dyn-DP-FedAvg.

In these experiments, the clipping was taken place at the local training process in order to bind the sensitivity of the gradients. The noise was injected at the global aggregation process, which was calibrated to the local clipping norms. The wavelet transform takes place in the global aggregation, reducing the noise variance that's been injected into the model updates in a similar manner to the sample-level approach. This brings our experimental results to a close, demonstrating that our assertions are consistent with the experimental and theoretical results.

\section{Conclusion}\label{sec:co}

In this paper, 
we proposed a modification to the vanilla DP-SGD and DP-FedAvg algorithms by using a Haar wavelet transform. 
This modification ensures the same $(\varepsilon, \delta)$ differential privacy and provides significantly higher utility than the vanilla algorithms and recent improvements upon them. 
We provided theoretical guarantees on improved noise variance bounds and demonstrated improved utility using three benchmark datasets. 
However, 
our approach is slightly more computationally expensive than the vanilla algorithms due to the increased number of computations. 
Although our method shows higher utility, 
the convergence rate of the model remains the same, 
which means that more training epochs are required for the model to converge compared to the traditional method. 
This is a result of our method converging to much greater accuracy. 
Therefore, 
in future work, 
we plan to investigate the effects of applying unequal noise to different coefficients of the Haar transform. \color{black}Also, while the current median clipping approach performs well in i.i.d. settings, future work could explore adaptive clipping techniques. which dynamically adjusts the clipping threshold during training. This approach would be particularly beneficial in non-i.i.d. settings, where client data distributions vary significantly.\color{black}


%

\ifCLASSOPTIONcaptionsoff
  \newpage
\fi



%
\bibliographystyle{IEEEtran}
\bibliography{main}
\balance

\newpage
\appendix

\subsection{Proof of theorem~\ref{lem_con} }
\label{sec:app1}
\color{black}
\textbf{Theorem: }
Let $\{\theta_t\}_{t \geq 0}$ be the sequence of model parameters generated by the DP-FedAvg via Wavelets algorithm (Algorithm \ref{dp-fed}). Assume that the local loss functions $L_k(\theta)$ are Lipschitz continuous and the gradient dissimilarity is bounded. Then, the expected difference between the global loss function $L(\theta)$ at iteration $t+1$ and the optimal loss $L(\theta^*)$ is bounded as follows:

\begin{equation}
    \begin{aligned}
        \mathbb{E}[L(\theta_{t+1})] - L(\theta^*) \leq \triangle_t\mathbb{E}[L(\theta_t)-L(\theta^*)]+c_t\\+\dfrac{\eta_t}{2}[-1+\lambda M\eta_t(\dfrac{B_1+K}{K}]||\sum^K_{k=1}\dfrac{d_k}{q_cd}\nabla\theta_{t}^k\widehat{L}_k(\theta_{t}, x_{i}^k)||^2 \\+ B_t\sum^{t-1}_{j=t_c+1}\eta_j^2||\sum^K_{k=1}\dfrac{d_k}{q_cd}\nabla\theta_{t}^k\widehat{L}_k(\theta_{t}, x_{i}^k)||^2.
    \end{aligned}
\end{equation}

where,

\begin{equation}
    \triangle_t = 1-\mu\eta_t,
\end{equation}
\begin{equation}
    c_t=\dfrac{\eta_tMB_2^2}{K}[\dfrac{\eta_t}{2}+\dfrac{M(K+1)}{K}\sum^{t-1}_{j=t_c+1}\eta_j^2,
\end{equation}
\begin{equation}
    B_t = \dfrac{\lambda(K+1)\eta_tM^2}{K^2}(B_1+N).
\end{equation}

\begin{proof}

    From algorithm~\ref{dp-fed}'s local update rule we know that
    \begin{equation}
    \label{eq:}
        \theta_{t+1}^k = \theta_t^k - \eta_t({g}^k_t(x_{i}^k)).
    \end{equation}

    where, $g^k_t$($x_{i}^k$) = $\nabla\theta_{t}^kL_k(\theta_{t}, x_{i}^k) $,

    also, we know the FedAvg step;
    \begin{equation}
    \theta_{t+1} = \sum_{k=1}^K \dfrac{d_k}{q_cd}\theta_{t+1}^k.
    \end{equation}

   Let, 
   \begin{equation}
   \begin{aligned}
   \widehat{g}_t = &\sum_{k=1}^K \dfrac{d_k}{q_cd}\nabla\theta_{t}^kL_k(\theta_{t}, x_{i}^k) \\&+ \mathcal{N}(0, \dfrac{2+\log_2 (m)}{2} (\sigma_{(Haar)})^2) .     
   \end{aligned}
   \end{equation}
    
Since we assumed $L_k(\theta)$ is Lipschitz continuous

    \begin{equation}
        L(\theta_{t+1})-L(\theta_t) \leq -\eta_t\langle\nabla L(\theta_t),\widehat{g}_t\rangle + \dfrac{\eta_t^2M}{2}||\widehat{g}_t||^2.
    \end{equation}

    By taking expectation on both sides of above inequality overs sampling of devices $q_c$, we get

    \begin{equation}
    \begin{aligned}
    \label{eq:100}
        \mathbb{E}[\mathbb{E}_{k in K} [L(\theta_{t+1})-L(\theta_t)]] \\\leq -\eta_t\mathbb{E}[\mathbb{E}_{k in K}[\langle\nabla L(\theta_t),\widehat{g}_t\rangle]] + \dfrac{\eta_t^2M}{2}\mathbb{E}[\mathbb{E}_{k in K}[||\widehat{g}_t||^2]].
    \end{aligned}
    \end{equation}

We define $g_t$ as the full gradient of the local objective function at iteration $t$ where $\widehat{g}_t$ is an unbiased estimator of the full gradient $g_t$. 

Given the definitions of $g_t$ and $\widehat{g}_t$
, and considering that mini-batches are selected independently and identically distributed (i.i.d.) at each local machine, it follows that

\begin{equation}
\begin{aligned}
    \mathbb{E} \left[ \| \widehat{g}_t - g_t \|^2 \right] = \mathbb{E} \left[ \left\| \frac{1}g_t^k \mathbb{E}_{k in K} \widehat{g}_t^k - \frac{1}g_t^k \mathbb{E}_{k in K} g_t^{k} \right\|^2 \right]
     =\\ \frac{1}{K^2} \mathbb{E} \left[ \mathbb{E}_{k in K} \| (\widehat{g}_t^k - g_t^{k}) \|^2 + \sum_{i \neq k} \langle \widehat{g}_t^{i} - g_t^{i}, \widehat{g}_t^k - g_t^{k} \rangle \right]
     =\\ \frac{1}{K^2} \mathbb{E}_{k in K} \mathbb{E} \left[ \| (\widehat{g}_t^k - g_t^{k}) \|^2 \right] + \frac{1}{K^2} \sum_{i \neq k} \mathbb{E} \left[ \langle \widehat{g}_t^k - g_t^{k}, \widehat{g}_t^{i} - g_t^{i} \rangle \right]
     \\\leq \frac{1}{K^2} \mathbb{E}_{k in K} \mathbb{E} \left[ \| (\widehat{g}_t^k - g_t^k) \|^2 \right] \\+ \frac{1}{K^2} \sum_{i \neq k} \langle \mathbb{E} \left[ \widehat{g}_t^k - g_t^k \right], \mathbb{E} \left[ \widehat{g}_t^{i} - g_t^{i} \right] \rangle
     \\\leq \frac{1}{K^2} \mathbb{E}_{k in K} \left[ B_1 \| g_t^k \|^2 + B_2^2 \right]
     = \frac{B_1}{K^2} \mathbb{E}_{k in K} \| g_t^k \|^2 + \frac{B_2^2}{g_t^k}.
\end{aligned}
\end{equation}

Next, considering the expectation over the random sampling of devices on both sides of the above equation, we derive:

\begin{equation}
\begin{aligned}
    \mathbb{E}_{k in K} \left[ \mathbb{E} \left[ \| \widehat{g}_t - g_t \|^2 \right] \right] \leq \mathbb{E}_{k in K} \left[ \frac{B_1}{K^2} \mathbb{E}_{k in K} \| g_t^k \|^2 + \frac{B_2^2}g_t^k \right]
    \\= \frac{B_1}{K^2} \mathbb{E}_{k in K} \left[ \mathbb{E}_{k in K} \| g_t^k \|^2 \right] + \frac{B_2^2}{g_t^k}
    \\= \frac{B_1}{K^2} K \sum_{K=1}^K \dfrac{d_k}{q_cd} \| g_t^k \|^2 + \frac{B_2^2}{g_t^k}.
\end{aligned}
\end{equation}

Now, we note that $\mathbb{E}[\widehat{g}_t^k] = g_t^k$, from which we have
\begin{equation}
\begin{aligned}
    \mathbb{E}[\|\widehat{g}_t\|^2] & = \mathbb{E}[\|\widehat{g}_t - \mathbb{E}[\widehat{g}_t]\|^2] + \|\mathbb{E}[\widehat{g}_t]\|^2 \\
    & = \mathbb{E}[\|\widehat{g}_t - g_t\|^2] + \|g_t\|^2 \\
    & \leq \frac{B_1}{K^2} \mathbb{E}_{k in K} \|g_t^k\|^2 + \frac{B_2^2}{K} + \left\|\frac{1}{K} \mathbb{E}_{k in K} g_t^k\right\|^2 \\
    & \leq \frac{B_1}{K^2} \mathbb{E}_{k in K} \|g_t^k\|^2 + \frac{B_2^2}{K} + \frac{1}{K} \mathbb{E}_{k in K} \|g_t^k\|^2 \\
    & = \left(\frac{B_1+K}{K^2}\right) \mathbb{E}_{k in K} \|g_t^k\|^2 + \frac{B_2^2}{K},
\end{aligned}
\end{equation}
where we used the fact that $\left\|\sum_{i=1}^{m} a_i\right\|^2 \leq m \sum_{i=1}^{m} \|a_i\|^2$.

Applying the assumption that gradient dissimilarity is bounded, the second term on the right-hand side of \ref{eq:100} can be upper-bounded as:

\begin{equation}
    \begin{aligned}
    \label{eq:101}
    \mathbb{E}[\mathbb{E}_{k in K}[\|\widehat{g}_t\|^2]] \\ \leq \left(\frac{B_1+K}{K^2}\right) \left[\sum_{k=1}^{K} \frac{d_k}{q_c d} \|\nabla\theta_{t}^k \widehat{L}_k(\theta_{t}, x_{i}^k)\|^2\right] + \frac{B_2^2}{K} \\
    \leq \lambda \left(\frac{B_1+K}{K^2}\right) \left\|\sum_{k=1}^{K} \frac{d_k}{q_c d} \nabla\theta_{t}^k \widehat{L}_k(\theta_{t}, x_{i}^k)\right\|^2 + \frac{B_2^2}{K}.
    \end{aligned}
\end{equation}

where, 

\begin{equation*}
\begin{aligned}
    \nabla\theta_{t}^k \widehat{L}_k(\theta_{t}, x_{i}^k)=\\\nabla\theta_{t}^kL_k(\theta_{t}, x_{i}^k) + \mathcal{N}(0, \dfrac{2+\log_2 (m)}{2} (\sigma_{(Haar)})^2),
\end{aligned}
\end{equation*}
and $\lambda$ is the upper bound over the weighted gradient diversity, i.e

\begin{equation}
    \dfrac{\sum_{k=1}^K \dfrac{d_k}{q_cd}||\nabla\theta_{t}^k\widehat{L}_k(\theta_{t}, x_{i}^k)||_2^2}{||\sum_{k=1}^K \dfrac{d_k}{q_cd}\nabla\theta_{t}^k\widehat{L}_k(\theta_{t}, x_{i}^k)||^2_2} \leq \lambda.
\end{equation}

now we move on to bound the first term of \ref{eq:100}.

Define \( \widehat{g}^{(t)} = \frac{1}{K} \sum_{k=1}^K \widehat{g}^{(t)}_k \) as the average of their local stochastic gradients at iteration \( t \). Thus, we have:

\begin{equation}
    \begin{aligned}
        -\mathbb{E}_{i \in L_n} \mathbb{E}_{k \in K} \left[ \left\langle \nabla L(\theta_t), \widehat{g}^{(t)} \right\rangle \right] 
        \\ = -\mathbb{E}_{i \in L_n} \mathbb{E}_{k \in K} \left[ \left\langle \nabla L(\theta_t), \frac{1}{K} \sum_{k=1}^K \widehat{g}^{(t)}_j \right\rangle \right].
    \end{aligned}
\end{equation}

Note that the order of taking expectation follows from the fact that devices are chosen first and thereafter the stochastic mini-batch gradients are computed, and noting the fact that devices are agnostic to the random selection at every communication round., we can write;

\begin{equation}\begin{aligned}
-\mathbb{E}_{i \in L_n} \mathbb{E}_{k \in K} \left[ \left\langle \nabla L(\theta_t), \frac{1}{K} \sum_{k=1}^K \widehat{g}^{(t)}_j \right\rangle \right]
\\= -\mathbb{E}_{k \in K} \mathbb{E}_{i \in L_n} \left[ \left\langle \nabla L(\theta_t), \frac{1}{K} \sum_{k=1}^K \widehat{g}^{(t)}_j \right\rangle \right]
\\= -\left\langle \nabla L(\theta_t), \mathbb{E}_{k in K} \left[ \frac{1}{K} \sum_{k=1}^K \mathbb{E}_t \left[\widehat{g}_j \right] \right] \right\rangle
\\= -\left\langle \nabla L(\theta_t), \mathbb{E}_{k in K} \left[ \frac{1}{K} \sum_{k=1}^K \nabla\theta_{t}^k \widehat{L}_k(\theta_{t}, x_{i}^k) \right] \right\rangle
\\= -\left\langle \nabla L(\theta_t), \frac{1}{K} \mathbb{E}_{k in K} \left[ \sum_{k=1}^K \nabla\theta_{t}^k \widehat{L}_k(\theta_{t}, x_{i}^k) \right] \right\rangle
\\= -\left\langle \nabla L(\theta_t), \frac{1}{K} \left[ K \sum_{k=1}^K \dfrac{d_k}{q_cd} \nabla\theta_{t}^k \widehat{L}_k(\theta_{t}, x_{i}^k) \right] \right\rangle
\\= -\left\langle \nabla L(\theta_t), \sum_{k=1}^K \dfrac{d_k}{q_cd} \nabla\theta_{t}^k \widehat{L}_k(\theta_{t}, x_{i}^k) \right\rangle.
\end{aligned} \end{equation}

Based on the identity \( 2\langle a, b \rangle = \|a\|^2 + \|b\|^2 - \|a - b\|^2 \), we can write;
\begin{equation*}\begin{aligned}
-\left\langle \nabla L(\theta_t), \sum_{k=1}^K \dfrac{d_k}{q_cd} \nabla\theta_{t}^k \widehat{L}_k(\theta_{t}, x_{i}^k) \right\rangle
\\= \frac{1}{2} \biggl[ -\|\nabla L(\theta_t)\|^2 - \left\|\sum_{k=1}^K \dfrac{d_k}{q_cd} \nabla\theta_{t}^k \widehat{L}_k(\theta_{t}, x_{i}^k) \right\|^2 \\+ \left\|\nabla L(\theta_t) - \sum_{k=1}^K \dfrac{d_k}{q_cd} \nabla\theta_{t}^k \widehat{L}_k(\theta_{t}, x_{i}^k) \right\|^2 \biggr]
\\= \frac{1}{2} \biggl[ -\|\nabla L(\theta_t)\|^2 - \left\|\sum_{k=1}^K \dfrac{d_k}{q_cd} \nabla\theta_{t}^k \widehat{L}_k(\theta_{t}, x_{i}^k) \right\|^2 \\+ \left\| \sum_{k=1}^K \dfrac{d_k}{q_cd} (\nabla\theta_{t}\widehat{L}_k(\theta_{t}, x_{i}^k) - \nabla\theta_{t}^k \widehat{L}_k(\theta_{t}, x_{i}^k)) \right\|^2 \biggr]
\end{aligned} \end{equation*} 

\begin{equation}\begin{aligned}
\leq \frac{1}{2} \biggl[ -\|\nabla L(\theta_t)\|^2 - \left\|\sum_{k=1}^K \dfrac{d_k}{q_cd} \nabla\theta_{t}^k \widehat{L}_k(\theta_{t}, x_{i}^k) \right\|^2 \\+ \sum_{k=1}^K \dfrac{d_k}{q_cd} \|\nabla\theta_{t}\widehat{L}_k(\theta_{t}, x_{i}^k) - \nabla\theta_{t}^k \widehat{L}_k(\theta_{t}, x_{i}^k)\|^2 \biggr].
\end{aligned} \end{equation} 

Since $L_k(\theta)$ is Lipschitz continuous, we have;

\begin{equation}\begin{aligned}
\frac{1}{2} \biggl[ -\|\nabla L(\theta_t)\|^2 - \left\|\sum_{k=1}^K \dfrac{d_k}{q_cd} \nabla\theta_{t}^k \widehat{L}_k(\theta_{t}, x_{i}^k) \right\|^2 \\+ \sum_{k=1}^K \dfrac{d_k}{q_cd} \|\nabla\theta_{t}\widehat{L}_k(\theta_{t}, x_{i}^k) - \nabla\theta_{t}^k \widehat{L}_k(\theta_{t}, x_{i}^k)\|^2 \biggr]
\\\leq \frac{1}{2} \biggl[ -\|\nabla L(\theta_t)\|^2 - \left\|\sum_{k=1}^K \dfrac{d_k}{q_cd} \nabla\theta_{t}^k \widehat{L}_k(\theta_{t}, x_{i}^k) \right\|^2 \\+ \sum_{k=1}^K \dfrac{d_k}{q_cd} M^2 \|\theta_t-\theta_t^k\|^2 \biggr].
\end{aligned} \end{equation} 

Therefore the first term of \ref{eq:100} can be bounded as. 

\begin{equation}
\label{eq:102}
    \begin{aligned}
        -\eta_t\mathbb{E}[\mathbb{E}_{k in K}[\langle\nabla L(\theta_t),\widehat{g}_t\rangle]] \\ \leq -\dfrac{\eta_t}{2}||\nabla L(\theta_t)||^2 - \dfrac{\eta_t}{2}||\sum_{k=1}^K \dfrac{d_k}{q_cd}\nabla\theta_{t}^k\widehat{L}_k(\theta_{t}, x_{i}^k)||^2\\ + \dfrac{\eta_tM^2}{2}\sum_{k=1}^K \dfrac{d_k}{q_cd}||\theta_t-\theta_t^k||^2. 
    \end{aligned}
\end{equation}

An immediate implication of the above equation is

\begin{equation*}
\begin{aligned}
    -\eta_t \mathbb{E} \left[ \mathbb{E}_{k in K} \left[ \left\langle \nabla L(\theta_t), \widehat{g}^{(t)} \right\rangle \right] \right] \\\leq -\frac{\eta_t}{2} \left\| \nabla L(\theta_t) \right\|^2 - \frac{\eta_t}{2} \left\| \sum_{k=1}^{K} \dfrac{d_k}{q_cd} \nabla\theta_{t}^k\widehat{L}_k(\theta_{t}, x_{i}^k) \right\|^2 \\+ \frac{\eta_t M^2}{2} \sum_{k=1}^{K} \dfrac{d_k}{q_cd} \left\| \theta_t - \theta_t^k \right\|^2 
\end{aligned}
\end{equation*}

\begin{equation}
\label{eq:103}
\begin{aligned}
    \leq -\mu \eta_t (L(\theta_t) - L(\theta^*)) - \frac{\eta_t}{2} \left\| \sum_{k=1}^{K} \dfrac{d_k}{q_cd} \nabla\theta_{t}^k\widehat{L}_k(\theta_{t}, x_{i}^k) \right\|^2 \\+\frac{\eta_t L^2}{2} \sum_{k=1}^{K} \dfrac{d_k}{q_cd} \left\| \theta_t - \theta_t^k \right\|^2,
\end{aligned}
\end{equation}
where the last inequality follows from the PL property.

We can also simplify the last term in RHS of \ref{eq:103}

Define \( t_c \triangleq \left\lfloor \frac{t}{|N|} \right\rfloor |N| \). Therefore, according to Algorithm \ref{dp-fed}, we have:

\begin{equation}
\theta_{t_c+1} = \frac{1}{K} \sum_{k=1}^K \theta^k_{t_c+1}
\end{equation}

The update rule of Algorithm \ref{dp-fed} can then be expressed as:

\begin{equation}
\begin{aligned}
\theta_t^k = \theta^k_{t-1} - \eta_{t-1} \widehat{g}^k_{t-1}
=\theta^k_{(t-2)} - \left(\eta_{t-2} \widehat{g}^k_{t-2} + \eta_{t-1} \widehat{g}^k_{(t-1)}\right)
\\= \theta_{t_c+1} - \sum_{j=t_c+1}^{t-1} \eta_j \widehat{g}^k_j,
\end{aligned}
\end{equation}

where the last equality follows from the update rule. Building on this, we now compute the average model as follows:

\begin{equation}
\theta_t = \theta_{t_c+1} - \frac{1}{K} \sum_{k=1}^K \sum_{j=t_c+1}^{t-1} \eta_j \widehat{g}_j^k
\end{equation}

Next, we aim to bound the term \( \mathbb{E}\|\theta_t - \theta_t^k\|^2 \) for \( t_c + 1 \leq t \leq t_c + E \) (where \( n \) represents the indices of local updates). We begin by relating this quantity to the variance between the stochastic gradient and the full gradient.

\begin{equation*}
\begin{aligned}
\mathbb{E}\left[\|\theta_{(t_c+n)} - \theta^k_{(t_c+n)}\|^2\right] \\= \mathbb{E}\left[\left\|\theta_{t_c+1} - \sum_{j=t_c+1}^{t-1} \eta_j \widehat{g}^k_j 
- \theta_{t_c+1} + \frac{1}{K} \sum_{k=1}^K \sum_{j=t_c+1}^{t-1} \eta_j \widehat{g}_j^{k}\right\|^2\right] \\
= \mathbb{E}\left[\left\|\sum_{j=1}^{n} \eta_{t_c+j} \widehat{g}^k_{t_c+j} - \frac{1}{K} \sum_{k=1}^K \sum_{j=1}^{n} \eta_{t_c+j} \widehat{g}_j^{t_c+j}\right\|^2\right] \\
\leq 2 \mathbb{E}\left[\left\|\sum_{j=1}^{n} \eta_{t_c+j} \widehat{g}^k_{t_c+j}\right\|^2\right] + 2 \mathbb{E}\left[\left\|\frac{1}{K} \sum_{k=1}^K \sum_{j=1}^{n} \eta_{t_c+j} \widehat{g}_j^{t_c+j}\right\|^2\right] \\
= 2 \left(\mathbb{E}\left[\left\|\sum_{j=1}^{n} \eta_{t_c+j} \widehat{g}^k_{t_c+j} - \mathbb{E}\left[\sum_{j=1}^{n} \eta_{t_c+j} \widehat{g}^k_{t_c+j}\right]\right\|^2\right] \right.\\
\quad + \left. \left\|\mathbb{E}\left[\sum_{j=1}^{n} \eta_{t_c+j} \widehat{g}^k_{t_c+j}\right]\right\|^2\right)
\quad + 2 \mathbb{E}\biggl[\biggl\|\frac{1}{K} \sum_{k=1}^K \sum_{j=1}^{n} \eta_{t_c+j} \widehat{g}_j^{t_c+j} \\
\end{aligned}
\end{equation*}

\begin{equation}
\begin{aligned}
- \mathbb{E}\left[\frac{1}{K} \sum_{k=1}^K \sum_{j=1}^{n} \eta_{t_c+j} \widehat{g}_j^{t_c+j}\right]\biggr\|^2\biggr] \\
\quad + \left\|\mathbb{E}\left[\frac{1}{K} \sum_{k=1}^K \sum_{j=1}^{n} \eta_{t_c+j} \widehat{g}_j^{t_c+j}\right]\right\|^2 \\
= 2 \mathbb{E}\left[\left\|\sum_{j=1}^{n} \eta_{t_c+j} \left(\widehat{g}^k_{t_c+j} - g^k_{t_c+j}\right)\right\|^2\right] + \left\|\sum_{j=1}^{n} \eta_{t_c+j} g^k_{t_c+j}\right\|^2 \\
\quad + 2 \mathbb{E}\left[\left\|\frac{1}{K} \sum_{k=1}^K \sum_{j=1}^{n} \eta_{t_c+j} \left(\widehat{g}_j^{t_c+j} - g_j^{t_c+j}\right)\right\|^2\right] \\
\quad + \left\|\frac{1}{K} \sum_{k=1}^K \sum_{j=1}^{n} \eta_{t_c+j} g_j^{t_c+j}\right\|^2,
\end{aligned}
\end{equation}

where the inequalities follow from the properties of convexity and unbiased estimation.

Given the Smoothness assumption and the i.i.d. sampling, we can express

\begin{equation}
\begin{aligned}
    \mathbb{E}\left[\|\theta_{(t_c+n)} - \theta^k_{(t_c+n)}\|^2\right] 
    \\\leq 2\mathbb{E}\Bigg(\biggl[\sum_{j=1}^{n} \eta^2_{t_c+j} \|g_{t_c+j}^k - g_{t_c+j}\|^2 \\+ \sum_{j \neq u \lor k \neq v} \left\langle \eta_j g_{t_c+j}^k - \eta_j g_{t_c+j}, \eta_u g_{t_c+u}^v - \eta_u g_{t_c+u} \right\rangle \\+ \biggl\| \sum_{j=1}^{n} \eta_{t_c+j} g_{t_c+j} \biggr\|^2\biggr] \nonumber
    + \frac{1}{K^2} \sum_{k \in K} \sum_{j=1}^{n} \eta^2_{t_c+j} \|g_{t_c+j}^j - g_{t_c+j}\|^2  \\+\dfrac{1}{K^2}\sum_{j \neq u \lor k \neq v} \left\langle \eta_j g_{t_c+j}^j - \eta_j g_{t_c+j}, \eta_u g_{t_c+u}^v - \eta_u g_{t_c+u} \right\rangle \\+ \left\| \frac{1}{K} \sum_{k \in K} \sum_{j=1}^{n} \eta_{t_c+j} g_{t_c+j}^j \right\|^2\Bigg) \nonumber  \\
    \leq 2\mathbb{E}\Bigg(\left[\sum_{j=1}^{n} \eta^2_{t_c+j} \|g_{t_c+j}^k - g_{t_c+j}\|^2 + n \sum_{j=1}^{n} \eta^2_{t_c+j} \|g_{t_c+j}\|^2 \right]\\+
    \frac{1}{K^2} \sum_{k \in K} \sum_{j=1}^{n} \eta^2_{t_c+j} \|g_{t_c+j}^j - g_{t_c+j}\|^2 \\+ \left\| \frac{1}{K} \sum_{k \in K} \sum_{j=1}^{n} \eta_{t_c+j} g_{t_c+j}^j \right\|^2\Bigg) \nonumber=\\
    2\mathbb{E}\Bigg(\left[\sum_{j=1}^{n} \eta^2_{t_c+j} \|g_{t_c+j}^k - g_{t_c+j}\|^2 + n \sum_{j=1}^{n} \eta^2_{t_c+j} \|g_{t_c+j}\|^2 \right] \nonumber\\
    + \frac{1}{K^2} \sum_{k \in K} \sum_{j=1}^{n} \|g_{t_c+j}^j - g_{t_c+j}\|^2 \\+ \frac{n}{K^2} \sum_{k \in K} \sum_{j=1}^{n} \eta^2_{t_c+j} g_{t_c+j}^j \|^2\Bigg) \nonumber=\\
    2\Bigg(\left[\sum_{j=1}^{n} \eta^2_{t_c+j} \mathbb{E} \|g_{t_c+j}^k - g_{t_c+j}\|^2 + n \sum_{j=1}^{n} \eta^2_{t_c+j} \mathbb{E} \|g_{t_c+j}\|^2 \right] \nonumber\\
    + \frac{1}{K^2} \sum_{k \in K} \sum_{j=1}^{n} \eta^2_{t_c+j} \mathbb{E} \|g_{t_c+j}^k - g_{t_c+j}\|^2 \\ + \frac{n}{K^2} \sum_{k \in K} \sum_{j=1}^{n} \eta^2_{t_c+j} \mathbb{E}\|g_{t_c+j}^j \|^2 \nonumber.
\end{aligned}
\end{equation}

Now using assumption 1;

\begin{equation*}
\begin{aligned}
\mathbb{E}\left[\|\theta_t - \theta^k_{t}\|^2\right] \leq 2 \biggl( \biggl[ \sum_{j=1}^{n} \eta_{t_c+j}^2 \left( B_1 \| g^k_{(t_c+j)} \|^2 + \frac{B_2^2}{K} \right)\\+ n \sum_{j=1}^{n} \eta_{t_c+j}^2 \| g^k_{(t_c+j)} \|^2 \biggr] +
\end{aligned}
\end{equation*}

\begin{equation}
\begin{aligned}
 \frac{1}{K^2} \sum_{k \in K} \sum_{j=1}^{n} \eta_{t_c+j}^2 \biggl( B_1 \| g_j^{(t_c+j)} \|^2 + \frac{B_2^2}{K} \biggr) \\+ \frac{n}{K^2} \sum_{k \in K} \sum_{j=1}^{n} \eta_{t_c+j}^2 \| g_j^{(t_c+j)} \|^2 \biggr)
\\=2 \biggl( \biggl[ \sum_{j=1}^{n} \eta_{t_c+j}^2 B_1 \| g^k_{(t_c+j)} \|^2 + \sum_{j=1}^{n}\eta_{t_c+j}^2\frac{B_2^2}{K}  +\\ n \sum_{j=1}^{n} \eta_{t_c+j}^2 \| g^k_{(t_c+j)} \|^2 \biggr] + \frac{1}{K^2} \sum_{k \in K} \sum_{j=1}^{n} \eta_{t_c+j}^2 B_1 \| g_j^{(t_c+j)} \|^2 \\+ \sum_{j=1}^{n}\eta_{t_c+j}^2\frac{B_2^2}{K^2} \biggr) + \frac{n}{K^2} \sum_{k \in K} \sum_{j=1}^{n} \eta_{t_c+j}^2 \| g_j^{(t_c+j)} \|^2 \biggr).
\end{aligned}
\end{equation}

Now taking summation over sampled clients, we obtain:

\begin{equation}\begin{aligned}
\mathbb{E} \sum_{k \in K} \|\theta_t - \theta^k_{t}\|^2 \leq 2 \biggl( \biggl[ \sum_{k \in K}\sum_{j=1}^{n} \eta_{t_c+j}^2 B_1 \| g^k_{(t_c+j)} \|^2 \\+ \sum_{j=1}^{n}\eta_{t_c+j}^2\frac{B_2^2}{K} + n \sum_{k \in K}\sum_{j=1}^{n} \eta_{t_c+j}^2 \| g^k_{(t_c+j)} \|^2 \biggr] \\+ \frac{1}{K} \sum_{k \in K} \sum_{j=1}^{n} \eta_{t_c+j}^2 B_1 \| g_j^{(t_c+j)} \|^2 \\+ \sum_{j=1}^{n}\eta_{t_c+j}^2\frac{B_2^2}{K^2} \biggr) + \frac{n}{K} \sum_{k \in K} \sum_{j=1}^{n} \eta_{t_c+j}^2 \| g_j^{(t_c+j)} \|^2 \biggr)
\\= 2\Bigg((\dfrac{K+1}{K})\sum_{k \in K}\sum_{j=1}^{n} \eta_{t_c+j}^2 B_1 \| g^k_{(t_c+j)} \|^2 \\+\sum_{j=1}^{n}\eta_{t_c+j}^2\frac{(K+1)B_2^2}{K} \\+n(\dfrac{K+1}{K})\sum_{k \in K}\sum_{j=1}^{n} \eta_{t_c+j}^2\| g_j^{(t_c+j)} \|^2
\\=2\Bigg( \left[(\dfrac{K+1}{K})(B_1+n)\right] 
\sum_{k \in K}\sum_{j=1}^{n} \eta_{t_c+j}^2\|
g_j^{(t_c+j)} \|^2\\+\sum_{j=1}^{n}\eta_{t_c+j}^2\dfrac{(K+1)B_2^2}{K}\Bigg)
\\\leq 2\Bigg(\dfrac{K+1}{K}\Bigg)\Bigg( \left[ B_1+N\right] (\sum_{k=t_c+1}^{t-2}\sum_{k \in K} 
\eta_k^2\| g_j^{(k)} \|^2 ) \\+\sum_{k=t_c+1}^{t-1}\eta_k^2\frac{B_2^2}{K}\Bigg).
\end{aligned}
\end{equation}

Here we note $n \leq |N|$

Finally,

\begin{equation*}
\begin{aligned}
\sum_{k\in K} \|\theta_t - \theta^k_{t}\|^2 \leq 
\end{aligned}
\end{equation*}

\begin{equation}
\begin{aligned}
2\Bigg(\dfrac{K+1}{K}\Bigg)\Bigg(\left[ B_1+N\right]\sum_{k=t_c+1}^{t-1}\eta_k^2\sum{k \in K}\|\nabla L_k(\theta_j^k) \|^2 \\+\sum_{k=t_c+1}^{t-1}\eta_k^2\frac{B_2^2}{K}\Bigg).
\end{aligned}
\end{equation}

Now using the upper bound over the weighted gradient diversity we obtain;

\begin{equation}
\label{eq:104}
    \begin{aligned}
        \mathbb{E}[\sum_{k=1}^K \dfrac{d_k}{q_cd}||\theta_t-\theta_t^k||^2] \\\leq- 2(\dfrac{K+1}{K})([B_1+N]\sum^{t-1}_{j=t_c+1}\eta_j^2\sum_{k=1}^K \dfrac{d_k}{q_cd}||\nabla\theta_{t}^k\widehat{L}_k(\theta_{t}, x_{i}^k)||^2\\+\sum^{t-1}_{j=t_c+1}\dfrac{\eta_j^2 B_2^2}{||L_n||})   \\   \leq 2(\dfrac{K+1}{K})(\lambda[B_1+N]\sum^{t-1}_{j=t_c+1}\eta_j^2||\sum_{k=1}^K \dfrac{d_k}{q_cd}\nabla\theta_{t}^k\widehat{L}_k(\theta_{t}, x_{i}^k)||^2\\+\sum^{t-1}_{j=t_c+1}\dfrac{\eta_j^2 B_2^2}{K}).
    \end{aligned}
\end{equation}

Where, $t_c = [\dfrac{t}{N}]N$.

by substituting \ref{eq:101},\ref{eq:102} and \ref{eq:104} into \ref{eq:100} we get

\begin{equation}
\label{eq:105}
    \begin{aligned}
        \mathbb{E}[L(\theta_{t+1})] - L(\theta^*) \leq (1-\mu\eta_t)\mathbb{E}[L(\theta_t)-L(\theta^*)]\\+\dfrac{M\eta_t^2B_2}{2K}+\dfrac{\eta_tM^2}{K}(\sum^{t-1}_{j=t_c+1}\eta_j^2\dfrac{(K+1)B_2^2}{K}) \\ \dfrac{\eta_t}{2}[-1+\dfrac{M\lambda\eta_t(B_1+K)}{K}]||\sum_{k=1}^K \dfrac{d_k}{q_cd}\nabla\theta_{t}^k\widehat{L}_k(\theta_{t}, x_{i}^k)||^2 \\ + \dfrac{\eta_tM^2(K+1)}{K^2}[\lambda(B_1+N)\sum^{t-1}_{j=t_c+1}\eta_j^2||\sum_{k=1}^K \dfrac{d_k}{q_cd}\nabla\theta_{t}^k\widehat{L}_k(\theta_{t}, x_{i}^k)||^2].
    \end{aligned}
\end{equation}

Let, 

\begin{equation}
    \triangle_t = 1-\mu\eta_t,
\end{equation}
\begin{equation}
    c_t=\dfrac{\eta_tMB_2^2}{K}[\dfrac{\eta_t}{2}+\dfrac{M(K+1)}{K}\sum^{t-1}_{j=t_c+1}\eta_j^2,
\end{equation}
\begin{equation}
    B_t = \dfrac{\lambda(K+1)\eta_tM^2}{K^2}(B_1+N).
\end{equation}

Now we can simplify \ref{eq:104} as 

\begin{equation}
\label{eq:105}
    \begin{aligned}
        \mathbb{E}[L(\theta_{t+1})] - L(\theta^*) \leq \triangle_t\mathbb{E}[L(\theta_t)-L(\theta^*)]+c_t\\+\dfrac{\eta_t}{2}[-1+\lambda M\eta_t(\dfrac{B_1+K}{K})]||\sum^K_{k=1}\dfrac{d_k}{q_cd}\nabla\theta_{t}^k\widehat{L}_k(\theta_{t}, x_{i}^k)||^2 \\+ B_t\sum^{t-1}_{j=t_c+1}\eta_j^2||\sum^K_{k=1}\dfrac{d_k}{q_cd}\nabla\theta_{t}^k\widehat{L}_k(\theta_{t}, x_{i}^k)||^2.
    \end{aligned}
\end{equation}

\end{proof}

\end{document}